\documentclass[sigconf]{acmart}
\AtBeginDocument{%
  \providecommand\BibTeX{{%
    \normalfont B\kern-0.5em{\scshape i\kern-0.25em b}\kern-0.8em\TeX}}}

\copyrightyear{2025}
\acmYear{2025}
\setcopyright{acmlicensed}
\acmConference[KDD '25]{Proceedings of the 31st ACM SIGKDD Conference on Knowledge Discovery and Data Mining V.1}{August 3--7, 2025}{Toronto, ON, Canada}
\acmBooktitle{Proceedings of the 31st ACM SIGKDD Conference on Knowledge Discovery and Data Mining V.1 (KDD '25), August 3--7, 2025, Toronto, ON, Canada}
\acmDOI{10.1145/3690624.3709186}
\acmISBN{979-8-4007-1245-6/25/08}

\usepackage{booktabs}
\usepackage{multirow}
\usepackage[normalem]{ulem}
\usepackage{graphicx}
\usepackage{subcaption}
\usepackage{tablefootnote}
\usepackage{xcolor}
\usepackage{diagbox}
\usepackage{mathtools}

\newcommand{\eat}[1]{}

\newtheorem{theorem}{Theorem}
\newtheorem{theoremBrief}{Theorem}
\newtheorem{assumption}{Assumption}
\newtheorem{lemma}{Lemma}

\begin{document}
%\title{GraphLoRA: Alleviating Feature and Structure Discrepancies for Graph Neural Network Transfer Learning}
\title{GraphLoRA: Structure-Aware Contrastive Low-Rank Adaptation for Cross-Graph Transfer Learning}

\author{Zhe-Rui Yang}
\affiliation{%
  \institution{CSE, Sun Yat-sen University, Guangzhou, China}
  \city{}
  \country{}
}
\affiliation{%
  \institution{AI Thrust, HKUST(GZ), Guangzhou, China}
  \city{}
  \country{}
}
\affiliation{%
  \institution{Guangdong Key Laboratory of Big Data Analysis and Processing, Guangzhou, China}
  \city{}
  \country{}
}
% \affiliation{%
%   \institution{The Hong Kong University of Science and Technology (Guangzhou)}
%   \city{Guangzhou}
%   \country{China}
% }
\email{yangzhr9@mail2.sysu.edu.cn}

\author{Jindong Han}
\affiliation{%
  \institution{EMIA, HKUST, Hong Kong, China}
  \city{}
  \country{}
}
\affiliation{%
  \institution{AI Thrust, HKUST(GZ), Guangzhou, China}
  \city{}
  \country{}
}
\email{jhanao@connect.ust.hk}

\author{Chang-Dong Wang*}
\affiliation{%
  \institution{CSE, Sun Yat-sen University, Guangzhou, China}
  \city{}
  \country{}
}
\affiliation{%
  \institution{Guangdong Key Laboratory of Big Data Analysis and Processing, Guangzhou, China}
  \city{}
  \country{}
}
\email{changdongwang@hotmail.com}

\author{Hao Liu*}
\affiliation{%
  \institution{AI Thrust, HKUST(GZ), Guangzhou, China}
  \city{}
  \country{}
}
\affiliation{%
  \institution{CSE, HKUST, Hong Kong, China}
  \city{}
  \country{}
}
\email{liuh@ust.hk}

\thanks{* Corresponding author.}

\renewcommand{\shortauthors}{Yang et al.}

\begin{abstract}
Graph Neural Networks (GNNs) have demonstrated remarkable proficiency in handling a range of graph analytical tasks across various domains, such as e-commerce and social networks. Despite their versatility, GNNs face significant challenges in transferability, limiting their utility in real-world applications. Existing research in GNN transfer learning overlooks discrepancies in distribution among various graph datasets, facing challenges when transferring across different distributions. How to effectively adopt a well-trained GNN to new graphs with varying feature and structural distributions remains an under-explored problem. Taking inspiration from the success of Low-Rank Adaptation~(LoRA) in adapting large language models to various domains, we propose GraphLoRA, an effective and parameter-efficient method for transferring well-trained GNNs to diverse graph domains. Specifically, we first propose a Structure-aware Maximum Mean Discrepancy (SMMD) to align divergent node feature distributions across source and target graphs. Moreover, we introduce low-rank adaptation by injecting a small trainable GNN alongside the pre-trained one, effectively bridging structural distribution gaps while mitigating the catastrophic forgetting. Additionally, a structure-aware regularization objective is proposed to enhance the adaptability of the pre-trained GNN to target graph with scarce supervision labels. Extensive experiments on eight real-world datasets demonstrate the effectiveness of GraphLoRA against fourteen baselines by tuning only $20\%$ of parameters, even across disparate graph domains. The code is available at \url{https://github.com/AllminerLab/GraphLoRA}.
\end{abstract}




\maketitle

\section{Introduction}
\eat{Graph neural networks (GNNs) have shown remarkable effectiveness in dealing with different types of graph data~\cite{DBLP:conf/sigir/ZhouTHZW21, DBLP:conf/kdd/WangHZZZL18}, making them widely applied in various graph-based tasks such as node classification~\cite{DBLP:conf/icml/ChenWHDL20}, link prediction~\cite{DBLP:journals/corr/abs-2010-16103}, and graph classification~\cite{DBLP:conf/ijcai/ZhouSHZZH20}. However, despite their strong ability to capture complex relationships in graphs, GNNs often encounter challenges concerning cross-task and cross-graph transferability. 

Most previous research in graph transfer learning has focused on cross-task transfer~\cite{DBLP:conf/kdd/SunZHWW22, DBLP:conf/www/LiuY0023, DBLP:conf/kdd/SunCLLG23}, which involves transferring a pre-trained GNN to a downstream task different from the pretext task used during pre-training. For instance, GPPT~\cite{DBLP:conf/kdd/SunZHWW22} introduced the graph prompt technique to bridge the gap between pre-training with a link prediction task and the downstream node classification task. Nevertheless, compared to cross-task transfer, cross-graph transfer in graph neural networks may be more challenging.}

Graph Neural Networks (GNNs) have emerged as a powerful tool for processing and analyzing graph-structured data, demonstrating exceptional performance across diverse domains (e.g., e-commerce~\cite{DBLP:conf/sigir/ZhouTHZW21}, social networks~\cite{DBLP:conf/kdd/WangHZZZL18}, and recommendation~\cite{DBLP:conf/icdm/YangHWLLW22, DBLP:journals/kais/LaiYDLW24}) for diverse tasks, such as node classification~\cite{DBLP:conf/icml/ChenWHDL20}, link prediction~\cite{DBLP:journals/corr/abs-2010-16103}, and graph classification~\cite{DBLP:conf/ijcai/ZhouSHZZH20}.

\begin{figure}[!t]
  \centerline{\includegraphics[width=0.9\linewidth]{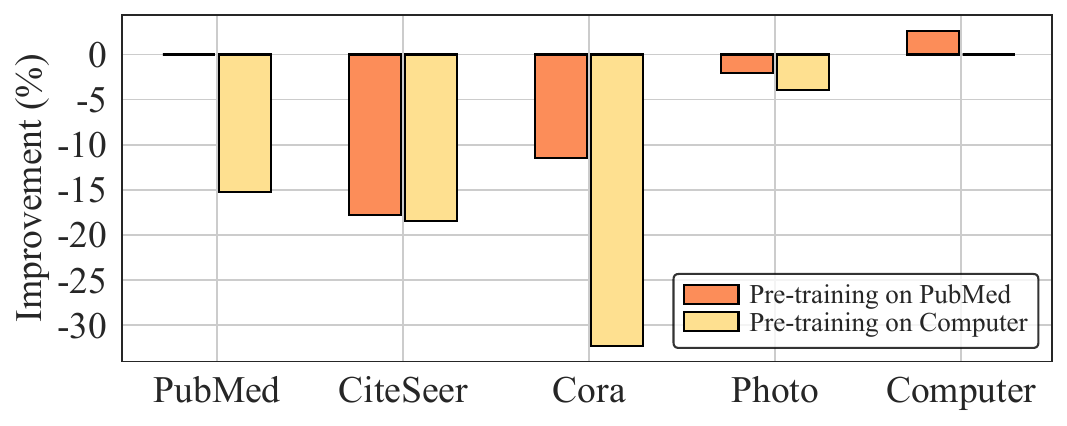}}
  \caption{Negative transfer occurs in cross-graph adaptation, where PubMed, CiteSeer, and Cora are citation networks, whereas Photo and Computer are co-purchase networks.}
  \label{fig:negative}
\end{figure}

Despite their superiority in capturing intricate graph relationships, GNNs heavily rely on graph labels, which are often insufficient in real-world scenarios~\cite{DBLP:conf/kdd/SunCLLG23}. Transfer learning is a common solution to address the issue of label sparsity~\cite{DBLP:journals/pieee/ZhuangQDXZZXH21}. However, GNNs face significant challenges in transfer learning due to substantial variations in underlying data distributions~\cite{DBLP:conf/kdd/CaoXYWZ0C023, DBLP:conf/nips/LiuLJ22}. 
Transferring a well-trained GNN to another graph typically yields suboptimal performance (i.e., negative transfer~\cite{DBLP:journals/corr/abs-2402-09834}) due to out-of-distribution issues. As depicted in \figurename~\ref{fig:negative}, negative transfer occurs in cross-graph adaptation, even when transferring between citation networks or co-purchase networks within the same domain.

To tackle this challenge, a common approach is to imbue GNNs with generalizable graph knowledge through the 'pre-train and fine-tune' paradigm~\cite{DBLP:conf/kdd/QiuCDZYDWT20, DBLP:conf/kdd/HanHAB21}. In this approach, it's crucial to integrate domain-specific knowledge while preserving the universal knowledge acquired during pre-training. Consequently, Parameter Efficient Fine-Tuning (PEFT) has garnered considerable attention for its ability to mitigate the risks of overfitting and catastrophic forgetting~\cite{DBLP:journals/corr/abs-2312-12148}. PEFT updates only a small portion of the parameters while keeping the remaining parameters frozen, thus mitigating the risk of forgetting the universal knowledge~\cite{DBLP:journals/corr/abs-2312-12148}.

For instance, Gui et al.~\cite{DBLP:conf/aaai/GuiYX24} propose a structure-aware PEFT method named G-Adapter, aimed at adapting pre-trained Graph Transformer Networks to various graph-based downstream tasks. Li et al.~\cite{DBLP:conf/aaai/LiH024} propose the AdapterGNN method, which applies the adapter to non-transformer GNN architectures. However, while these methods focus on incorporating PEFT into GNNs, they lack specific mechanisms to address discrepancies in distribution among different graphs, such as variations in node features and graph structures. As a result, they encounter challenges when applied to graphs with varying distributions. 

How to effectively adapt well-trained GNNs to graphs with different distributions remains an under-explored problem, posing a non-trivial task due to three major challenges.
(1)~\emph{Cross-graph feature discrepancy.} The node feature distributions between source and target graphs can vary significantly, impeding the transferability of pre-trained GNNs. 
For example, attributes in academic citation networks (e.g., PubMed~\cite{DBLP:conf/icml/YangCS16}) differ greatly from those in e-commerce co-purchase networks (e.g., Computer~\cite{DBLP:journals/corr/abs-1811-05868}).
(2)~\emph{Cross-graph structural discrepancy.} The structural characteristics of source and target graphs are often diverged. For instance, academic citation networks typically exhibit higher density and consists of more cyclic motifs compared to e-commerce networks~\cite{DBLP:journals/jsea/WuHL08, DBLP:journals/corr/abs-1012-4050}.
(3)~\emph{Target graph label scarcity.} 
\eat{The effectiveness of pre-trained GNNs often relies on sufficient labels on target graphs, which is often unavailable. How to effectively adapting pre-trained GNNs to target graphs with limited labels is another challenge.}
The effectiveness of pre-trained GNNs often relies on sufficient labels on target graphs, which are not always available in the real-world. For instance, in social networks, typically only a small fraction of high-degree nodes are labeled~\cite{DBLP:journals/corr/abs-2401-10394}. 

To this end, in this paper, we present GraphLoRA, a structure-aware low-rank contrastive adaptation method for effective transfer learning of GNNs in cross-graph scenarios.
Specifically, we first introduce a Structure-aware Maximum Mean Discrepancy~(SMMD) to minimize the feature distribution discrepancy between source and target graphs.
Moreover, inspired by the success of Low-Rank Adaptation~(LoRA)~\cite{DBLP:conf/iclr/HuSWALWWC22} in adapting large language models to diverse natural language processing tasks and domains~\cite{DBLP:conf/emnlp/LongWP23, DBLP:conf/recsys/BaoZZWF023}, we construct a small trainable GNN alongside the pre-trained one coupled with a tailor-designed graph contrastive loss to mitigate structural discrepancies.
Additionally, we develop a structure-aware regularization objective by harnessing local graph homophily to enhance the model adaptability with scarce labels in the target graph.

The main contributions of this work are summarized as follows:
\begin{itemize}
  \item We propose a novel strategy for measuring feature distribution discrepancy in graph data, which incorporates the graph structure into the measurement process.
  \item We propose GraphLoRA, a novel method tailored for cross-graph transfer learning. The low-rank adaptation network, coupled with graph contrastive learning, efficiently incorporates structural information from the target graph, mitigating catastrophic forgetting and addressing structural discrepancies across graphs. Furthermore, we theoretically demonstrate that GraphLoRA possesses robust representational capabilities, facilitating effective cross-graph transfer.
  \item We propose a structure-aware regularization objective to enhance the adaptability of pre-trained GNN to target graphs, particularly in contexts with limited label availability.
  \item Extensive experiments on real-world datasets demonstrate the effectiveness of our proposed method by tuning a small fraction of parameters, even cross disparate graph domains.
\end{itemize}

\section{Related Work}
\subsection{Graph Transfer Learning}
Graph transfer learning involves pre-training a GNN and applying it to diverse tasks or datasets. Common techniques in graph transfer learning include multi-task learning~\cite{DBLP:conf/nips/HwangPKKHK20, DBLP:conf/kdd/HanHAB21}, multi-network learning~\cite{DBLP:conf/kdd/Jiang21, DBLP:conf/www/NiCLCCX018}, domain adaptation~\cite{DBLP:journals/tkde/DaiWXSW23, DBLP:conf/icml/LiuLFTZQL23}, and pre-train fine-tune approaches~\cite{DBLP:conf/kdd/QiuCDZYDWT20, DBLP:conf/ijcai/ZhangXHRB22, DBLP:conf/aaai/LiH024, DBLP:journals/corr/abs-2310-07365}. However, multi-task learning, multi-network learning and domain adaptation are typically employed for cross-task transfer or necessitate direct relationships between the source and target graphs, which is not applicable to our problem~\cite{DBLP:conf/kdd/HanHAB21, DBLP:conf/www/NiCLCCX018, DBLP:conf/icml/LiuLFTZQL23}. Therefore, we focus on pre-train and fine-tune techniques, involving pre-training a GNN on the source graph and subsequently fine-tuning it on the target graph. For instance, GCC~\cite{DBLP:conf/kdd/QiuCDZYDWT20}, GCOPE~\cite{DBLP:journals/corr/abs-2402-09834}, and GraphFM~\cite{lachi2024graphfm} focus on pre-training to develop more general GNNs. In contrast, DGAT~\cite{DBLP:journals/corr/abs-2402-08228} is dedicated to designing architectures that perform better in out-of-distribution scenarios. GTOT~\cite{DBLP:conf/ijcai/ZhangXHRB22}, AdapterGNN~\cite{DBLP:conf/aaai/LiH024}, and GraphControl~\cite{DBLP:journals/corr/abs-2310-07365} focus on fine-tuning, aiming to adapt pre-trained GNNs to various graphs. Most relevant to our work is GraphControl, which freezes the pre-trained GNN and utilizes information from the target graph as a condition to fine-tune the newly added ControlNet for cross-domain transfer. In contrast to GraphControl, our method aligns the node feature distributions to facilitate the transfer of the pre-trained GNN, rather than using node attributes as conditions. Additionally, we leverage graph contrastive learning to facilitate knowledge transfer and utilize graph structure knowledge to enhance the adaptability of the pre-trained GNN.

\begin{figure*}[!t]
  \centerline{\includegraphics[width=0.9\linewidth]{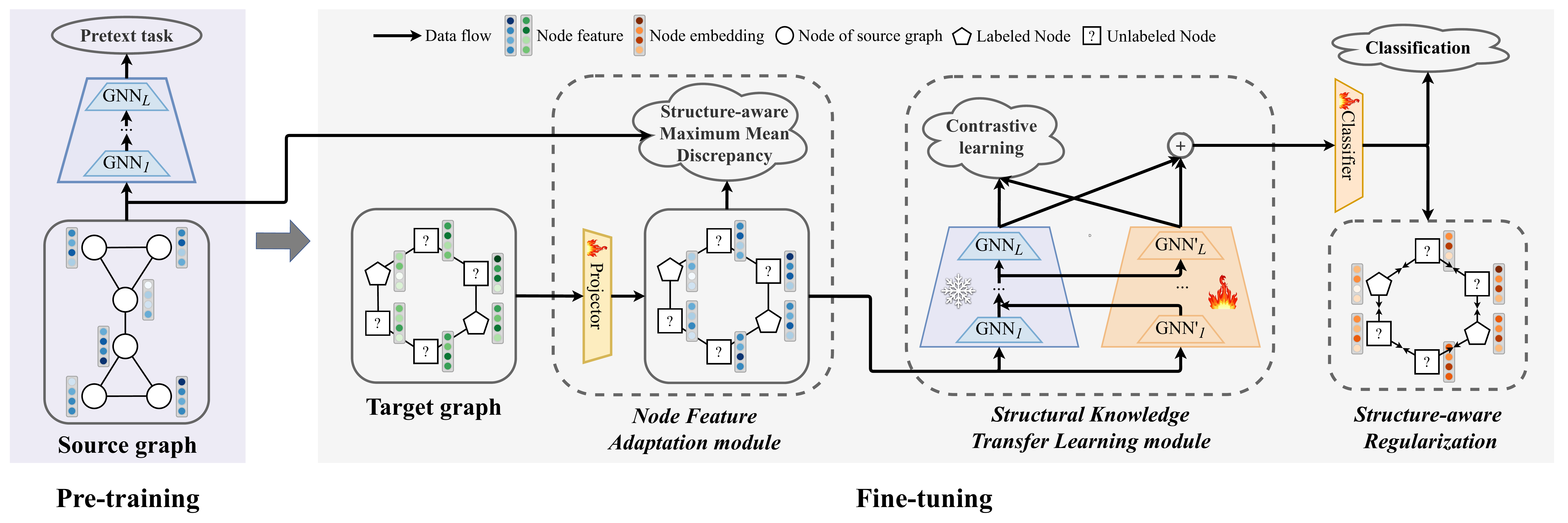}}
  \caption{The framework of GraphLoRA: Fine-tuning the pre-trained GNN for the target graph. The node feature adaptation and structural knowledge transfer learning modules are designed to alleviate feature and structural discrepancies, respectively. Furthermore, the structure-aware regularization objective is crafted to enhance the adaptability of the pre-trained GNN.}
  \label{fig:framework}
\end{figure*}

\subsection{Parameter-Effcient Fine-Tuning}
"Pre-train and fine-tune" has emerged as the predominant paradigm in transfer learning~\cite{DBLP:conf/kdd/SunCLLG23}. Despite its success, full fine-tuning is frequently inefficient and susceptible to challenges like overfitting and catastrophic forgetting~\cite{DBLP:conf/aaai/GuiYX24, DBLP:conf/aaai/LiH024, DBLP:journals/corr/abs-2312-12148}. In recent years, PEFT has emerged as an alternative, effectively mitigating these issues while achieving comparable performance~\cite{DBLP:journals/corr/abs-2312-12148}. PEFT methods update only a small portion of parameters while keeping the remaining parameters frozen. For instance, Adapter tuning~\cite{DBLP:conf/icml/HoulsbyGJMLGAG19, DBLP:conf/emnlp/LinMF20} inserts trainable adapter modules into the model, while Prompt tuning~\cite{DBLP:conf/emnlp/LesterAC21, DBLP:conf/acl/LiL20, DBLP:journals/corr/abs-2103-10385} inserts trainable parameters into the input or hidden states of the model. BitFit~\cite{DBLP:conf/acl/ZakenGR22} updates only the bias terms in the model, while LoRA~\cite{DBLP:conf/iclr/HuSWALWWC22} applies low-rank decomposition to reduce the number of trainable parameters. Recently, some efforts have been made to introduce PEFT into GNNs. For example, methods like GPPT~\cite{DBLP:conf/kdd/SunZHWW22}, GPF~\cite{DBLP:conf/nips/FangZYWC23}, GraphPrompt~\cite{DBLP:conf/www/LiuY0023}, and ProG~\cite{DBLP:conf/kdd/SunCLLG23} utilize prompt tuning to adapt pre-trained GNNs across diverse downstream tasks. G-Adapter~\cite{DBLP:conf/aaai/GuiYX24} adapts pre-trained Graph Transformer Networks for various graph-based downstream tasks, whereas AdapterGNN~\cite{DBLP:conf/aaai/LiH024} applies adapters to non-transformer GNN architectures.

\section{Preliminaries}
\subsection{Notations and Background}
In this paper, we utilize the notation $\mathcal{G}=\left( \mathcal{V}, \mathcal{E} \right)$ to denote a graph, where $\mathcal{V} = \{v_1, v_2, \cdots, v_N \}$ represents the node set with $N$ nodes in the graph, and $\mathcal{E} = \{(v_i, v_j)|v_i, v_j \in \mathcal{V}\}$ represents the edge set in the graph. The feature matrix is denoted as $\mathbf{X} \in \mathbb{R}^{N \times d}$, where $\mathbf{x}_i \in \mathbb{R}^d$ represents the node feature of $v_i$, and $d$ is the dimension of features. Furthermore, the adjacency matrix of the graph is denoted as $\mathbf{A} \in \{0, 1\}^{N \times N}$, where $\mathbf{A}_{i,j} = 1$ iff $(v_i, v_j) \in \mathcal{E}$. To avoid confusion, we employ the superscripts $s$ and $t$ in this paper to distinguish between the source and target graphs.  

\textbf{Graph neural networks}. A major category of GNNs is message-passing neural networks (MPNNs)~\cite{DBLP:conf/icml/GilmerSRVD17}. MPNNs follow the "propagate-transform" paradigm, which can be described as follows:
\begin{align}
    \bar{\boldsymbol{h}}_{s}^{l}&=\mathrm{Propagate}_l\left( \left\{ \boldsymbol{h}_{t}^{l-1}|v_t\in \mathcal{N} \left( v_s \right) \right\} \right), \\
    \boldsymbol{h}_{s}^{l}&=\mathrm{Transform}_l\left( \boldsymbol{h}_{s}^{l-1},\bar{\boldsymbol{h}}_{s}^{l} \right),
\end{align}
where $\mathcal{N} \left( v_s \right)$ denotes the neighboring node set of node $v_s$, $\boldsymbol{h}_{s}^{l}$ represents the node embedding of node $v_s$ in the $l$-th layer, and $\bar{\boldsymbol{h}}_{s}^{l}$ denotes the aggregated representation from neighboring nodes. 

\textbf{Low-Rank Adaptation~(LoRA)}.
LoRA~\cite{DBLP:conf/iclr/HuSWALWWC22} is a widely used PEFT methods, designed to efficiently fine-tune LLMs across tasks and domains. Compared to the Adapter method, LoRA provides better performance without introducing additional inference latency~\cite{DBLP:conf/iclr/HuSWALWWC22}.
Specifically, for a pre-trained weight matrix $\boldsymbol{W}_{0} \in \mathbb{R}^{m \times n}$, its weight update is expressed through a low-rank decomposition, given by $\boldsymbol{W}_0+\bigtriangleup \boldsymbol{W}=\boldsymbol{W}_0+\boldsymbol{BA}$, where $\boldsymbol{B} \in \mathbb{R}^{m \times r}$, $\boldsymbol{A} \in \mathbb{R}^{r \times n}$, and the rank $r\ll \min \left( m,n \right) $. During fine-tuning, the pre-trained weight matrix $\boldsymbol{W}_{0}$ is frozen, while $\boldsymbol{B}$ and $\boldsymbol{A}$ are tunable. 
The low-rank adaptation strategy effectively reduces the number of parameters requiring fine-tuning while maintaining high model quality without introducing inference latency. Notably, by sharing the vast majority of model parameters, it enables quick task switching and allows the pre-trained model to be applied to various tasks and domains.

\subsection{Problem Statement}
\eat{In this paper, we assume that a GNN has already been pre-trained on the source graph, utilizing methods such as graph contrastive learning, graph autoencoders, and so on~\cite{DBLP:journals/pami/XieXZWJ23, DBLP:journals/tkde/LiuJPZZXY23}. The objective is to fine-tune this pre-trained GNN on the target graph, with the goal of achieving superior performance on the target graph. Specifically, we focus on the task of node classification on graphs in this study.

Formally, assuming the pre-trained GNN is $g_{\theta}$ with parameters $\theta$. The node classifier $f_{\varTheta}\left( \cdot \right) =h_{\varTheta}\circ g_{\theta}\left( \cdot \right) $ is utilized to classify the nodes of the target graph $\mathcal{G}^t$, where $h_{\varTheta}$ is the tunable module, and $\varTheta$ are tunable parameters. The optimization goal for the node classifier fine-tuning is as follows:
\begin{align}
    f_{\varTheta}^{\star}=\underset{\varTheta}{\text{argmin}}\,\,\mathcal{L} \left( f_{\varTheta}\left( \mathbf{X}^t, \mathbf{A}^t \right) , Y_{train} \right),  
\end{align}
where $Y_{train}$ represents training labels, and $\mathcal{L}$ denotes the loss function of the fine-tuning stage. Finally, the optimal node classifier $f_{\varTheta}^{\star}$ is used to classify the testing nodes on the target graph.}

Given a GNN $g_{\theta}$ pre-trained on the source graph $\mathcal{G}^s$, our goal is to obtain an optimal GNN $f_{\varTheta}^{\star}$ for the target graph $\mathcal{G}^t$,

\begin{align}
    f_{\varTheta}^{\star}=\underset{\varTheta}{\text{argmin}}\,\,\mathcal{L} \left( f_{\varTheta}\left( \mathbf{X}^t, \mathbf{A}^t \right) , Y_{train} \right),  
\end{align}
where $Y_{train}$ denotes training labels, and $\mathcal{L}$ is the model tuning loss function. $f_{\varTheta}\left( \cdot \right) =h_{\varTheta}\circ g_{\theta}\left( \cdot \right)$, with $h_{\varTheta}$ as the tunable module and $g_{\theta}$ frozen. We focus on the node classification task in this work.

\section{Methodology}
\subsection{Framework Overview}
The overall framework of the model is illustrated in~\figurename~\ref{fig:framework}.
Firstly, a node feature adaptation module is designed to perform feature mapping on the target graph. Within this module, we propose a Structure-aware Maximum Mean Discrepancy strategy to minimize discrepancy in node feature distributions between the source and target graphs. After that, we introduce a structural knowledge transfer learning module to mitigate structural disparity. Taking inspiration from LoRA, we apply low-rank adaptation to the pre-trained GNN. This can be seen as incorporating an additional GNN $g_{\phi}^{'}$ with the same architecture, but utilizing the weight updates as its parameters. Additionally, we utilize graph contrastive learning to facilitate knowledge transfer. During this process, we freeze the weights of $g_{\theta}$ and fine-tune $g_{\phi}^{'}$. To enhance the adaptability of the pre-trained GNN to scenarios with scarce labels, a structure-aware regularization objective is proposed to effectively leverage the structural information of the target graph. Finally, we employ multi-task learning to optimize multiple objectives.

\subsection{Node Feature Adaptation}
Previous works on transfer learning have suggested that minimizing the discrepancy in feature distributions between the source and target datasets is crucial for effective knowledge transfer~\cite{DBLP:journals/ml/Ben-DavidBCKPV10, DBLP:journals/pieee/ZhuangQDXZZXH21}. To achieve this, a projector is designed to perform feature mapping on the node features of the target graph, as follows: 
\begin{align}
\boldsymbol{z}_{i}^{t}=p\left( \boldsymbol{x}_{i}^{t};\boldsymbol{\omega} \right) ,   
\end{align}
where $p\left( \cdot; \boldsymbol{\omega} \right)$ is the projector with parameters $\boldsymbol{\omega}$, and $\boldsymbol{z}_{i}^{t} \in \mathbb{R}^{d^s}$.

To optimize the projector, our goal is to minimize the discrepancy in feature distributions between $\boldsymbol{x}^{s}$ and $\boldsymbol{z}^{t}$. In the realm of domain adaptation, a commonly employed metric to quantify the dissimilarity between two probability distributions is the Maximum Mean Discrepancy (MMD)~\cite{DBLP:conf/ismb/BorgwardtGRKSS06, DBLP:journals/jmlr/GrettonBRSS12}. 
The fundamental principle underlying MMD is to measure the dissimilarity between two probability distributions by comparing their means in a high-dimensional feature space. Specifically, it can be expressed as follows:

\begin{equation}
\begin{split}
    \mathcal{L} _{mmd}= &
    \small{\frac{1}{\left( N^t \right) ^2}\sum_{i=1}^{N^t}{\sum_{i'=1}^{N^t}{k\left( \boldsymbol{z}_{i}^{t},\boldsymbol{z}_{i'}^{t} \right)}}} - 
    \small{\frac{2}{N^tN^s}\sum_{i=1}^{N^t}{\sum_{j=1}^{N^s}{k\left( \boldsymbol{z}_{i}^{t},\boldsymbol{x}_{j}^{s} \right)}}} \\ + &
    \small{\frac{1}{\left( N^s \right) ^2}\sum_{j=1}^{N^s}{\sum_{j'=1}^{N^s}{k\left( \boldsymbol{x}_{j}^{s},\boldsymbol{x}_{j'}^{s} \right)}}}, 
\end{split}
\end{equation}
where $k(\cdot, \cdot)$ denotes the kernel function. 

For the optimization of $\mathcal{L} _{mmd}$, we can observe that the first term of $\mathcal{L} _{mmd}$ maximizes the distance between node features in the target graph, while the second term minimizes the distance between node features of the source and target graphs. The third term denotes a constant. However, the node features in graph data are not independently and identically distributed ($i.i.d.$), $i.e.$, exhibiting correlation with the graph structure. For instance, neighboring nodes tend to exhibit similar features, which is overlooked by $\mathcal{L} _{mmd}$. Therefore, it is crucial to consider the graph structure when minimizing the discrepancy in feature distributions between $\boldsymbol{x}^{s}$ and $\boldsymbol{z}^{t}$. This aids in retaining the structural information in node features, such as homophily, during feature mapping.

Specifically, we introduce Structure-aware Maximum Mean Discrepancy (SMMD), which incorporates graph structure into the measurement of distribution discrepancy. In particular, smaller weights are assigned to node pairs with stronger connections, and larger weights are assigned to node pairs with weaker connections.

First, it is crucial to quantify the strength of relationships between node pairs. Since the adjacency matrix only provides a local perspective on the graph structure~\cite{DBLP:conf/icml/HassaniA20}, we utilize the graph diffusion technique to transform the adjacency matrix into a diffusion matrix~\cite{DBLP:conf/nips/KlicperaWG19, DBLP:conf/kdd/LiWXL23}. The diffusion matrix allows us to evaluate the strength of relationships between node pairs from a global perspective, facilitating the discovery of potential connections between node pairs and preserving them during feature mapping. Specifically, the diffusion matrix is defined as: 
\begin{align}
    \boldsymbol{S}=\sum_{r=0}^{\infty}{\psi _r\boldsymbol{T}^r}. 
\end{align}
where $\boldsymbol{T}$ represents the transition matrix, which is related to the adjacency matrix $\boldsymbol{A}^t$, and $\psi_r$ represents the weight coefficient. We utilize a popular variant of the diffusion matrix, Personalized PageRank (PPR)~\cite{page1998pagerank}, which employs $\boldsymbol{T}=\boldsymbol{A}^{t}\boldsymbol{D}^{-1}$ and $\psi _r=\alpha \left( 1-\alpha \right) ^r$, where $\boldsymbol{D}$ denotes the diagonal degree matrix, i.e. $\boldsymbol{D}_{i,i}=\sum_{j=1}^{N^t}{\boldsymbol{A}_{i,j}^{t}}$, and $\alpha \in \left( 0,1 \right) $ represents the teleport probability. The elements $\boldsymbol{S}_{i,j}$ in the obtained diffusion matrix $\boldsymbol{S}$ indicate the strength of relationships between node $v_i^t$ and node $v_j^t$. For PPR, the diffusion matrix has the closed-form expression:
\begin{align}
    S=\alpha(I-(1-\alpha)D^{-1/2}AD^{-1/2})^{-1}.
\end{align}

Finally, we define the Structure-aware Maximum Mean Discrepancy loss function as follows:
\begin{align}
    \mathcal{L}_{smmd} = & \frac{1}{\sum_{i=1}^{N^t}\sum_{i'=1}^{N^t}\gamma_{i,i'}}\sum_{i=1}^{N^t}\sum_{i'=1}^{N^t}\gamma_{i,i'}k\left(\boldsymbol{z}_{i}^{t},\boldsymbol{z}_{i'}^{t}\right) \\
    & - \frac{2}{N^tN^s}\sum_{i=1}^{N^t}\sum_{j=1}^{N^s}k\left(\boldsymbol{z}_{i}^{t},\boldsymbol{x}_{j}^{s}\right) + \frac{1}{(N^s)^2}\sum_{j=1}^{N^s}\sum_{j'=1}^{N^s}k\left(\boldsymbol{x}_{j}^{s},\boldsymbol{x}_{j'}^{s}\right), \nonumber \\
    \gamma_{i,i'} = & \log\left(1+\frac{1}{\boldsymbol{S}_{i,i'}}\right),
\end{align}
where $\mathcal{L} _{smmd}$ incorporates graph structure during computation.

\subsection{Structural Knowledge Transfer Learning}

Recent study~\cite{DBLP:conf/kdd/CaoXYWZ0C023} suggests that the disparity in graph structure between the source and target graphs impedes the transferability of pre-trained GNNs. 
Straightforward approaches such as full parameter fine-tuning of pre-trained GNNs may also lead to additional issues such as catastrophic forgetting~\cite{DBLP:conf/eacl/PfeifferKRCG21}.
Consequently, effectively transferring the pre-trained GNN to target graphs becomes a formidable challenge when there is a significant discrepancy in graph structure. 

Drawing inspiration from the success of LoRA across various tasks and domains, we propose incorporating adaptation for pre-trained weights, as depicted in \figurename~\ref{fig:framework}. During fine-tuning, we freeze the pre-trained weights while allowing newly added parameters to be tunable. From another perspective, it can be seen as maintaining the network architecture and parameters of the pre-trained GNN while introducing an additional GNN with the same architecture and utilizing the weight updates as its parameters.

Formally, let $g_{\phi}^{'}$ represents the newly added GNN, $GNN_l\left( \cdot;\boldsymbol{W}^l \right)$ and $GNN_{l}^{'}\left( \cdot;\bigtriangleup \boldsymbol{W}^l \right)$ denote the $l$-th layer of $g_{\theta}$ and $g_{\phi}^{'}$, respectively, where $\boldsymbol{W}^l,\bigtriangleup \boldsymbol{W}^l\in \mathbb{R}^{d^{l-1}\times d^l}$ are parameter matrices. The output of GNN at the $l$-th layer is modified from $\boldsymbol{H}^l=GNN_l\left( \boldsymbol{H}^{l-1};\boldsymbol{W}^l \right)$ to $\boldsymbol{H}^l=GNN_l\left( \boldsymbol{H}^{l-1};\boldsymbol{W}^l \right) +GNN_{l}^{'}\left( \boldsymbol{H}^{l-1};\bigtriangleup \boldsymbol{W}^l \right)$, where $\boldsymbol{H}^0=\boldsymbol{Z}^t$, and $\boldsymbol{Z}^t$ represents the feature-mapped node feature matrix. Let $\boldsymbol{H}=g_{\theta}\left( \boldsymbol{Z}^t \right)$ and $\boldsymbol{H}^{'}=g_{\phi}^{'}\left( \boldsymbol{Z}^t \right)$ represent the output of $g_{\theta}$ and $g_{\phi}^{'}$, respectively.
%Transfer learning is commonly employed in scenarios with scarce labels. 
We apply low-rank decomposition to the weight update to reduce the number of tunable parameters. Specifically, the output of the $l$-th layer of GNN is represented as follows:
\begin{align}
    \boldsymbol{H}^l=GNN_l\left( \boldsymbol{H}^{l-1};\boldsymbol{W}^l \right) +GNN_{l}^{'}\left( \boldsymbol{H}^{l-1};\boldsymbol{W}_B^l \boldsymbol{W}_A^l \right),
\end{align}
where $\boldsymbol{W}_B^l \in \mathbb{R}^{d^{l-1}\times r}$, $\boldsymbol{W}_A^l \in \mathbb{R}^{r \times d^l}$, and the rank $r\ll \min \left( d^{l-1},d^{l} \right) $.

The advantages of the above design are two-fold. First, the pre-trained GNN preserves general structural knowledge from the source graph, while the newly added one serves to incorporate specific structural information from the target graph, jointly facilitating downstream tasks. Keeping the pre-trained weights frozen during fine-tuning also mitigates the issue of catastrophic forgetting. Second, since the target graph may suffer from label scarcity, the low-rank decomposition reduces the number of tunable parameters to update and thus mitigates potential overfitting issues.

Unlike vanilla LoRA, which fine-tunes the network based on downstream task objectives, we further introduce graph contrastive learning to facilitate structural knowledge transfer. Specifically, we consider the embeddings obtained by two GNNs (frozen and tunable ones) for the same node as positive samples, while treating the embeddings for different nodes as negative samples. Furthermore, to enhance the learning effectiveness of node embeddings, we incorporate label information into graph contrastive learning. This involves treating the embeddings of nodes belonging to the same category as positive samples and considering the embeddings of nodes from different categories as negative samples.

Formally, the graph contrastive learning loss~\cite{DBLP:journals/corr/abs-1807-03748} is defined as
\begin{small}
\begin{align}
    \mathcal{L} _{cl}=-\sum_{i=1}^{N^t}{\sum_{y_i=y_k}{\log \frac{e^{\rho \left( \boldsymbol{h}_i,\boldsymbol{h}_{i}^{'} \right) /\tau}+\varepsilon e^{\rho \left( \boldsymbol{h}_i,\boldsymbol{h}_{k}^{'} \right) /\tau}}{e^{\rho \left( \boldsymbol{h}_i,\boldsymbol{h}_{i}^{'} \right) /\tau}+\sum_{j\ne i}{e^{\rho \left( \boldsymbol{h}_i,\boldsymbol{h}_{j}^{'} \right) /\tau}}+\sum_{j\ne i}{e^{\rho \left( \boldsymbol{h}_i,\boldsymbol{h}_{j} \right) /\tau}}}}},
\end{align}
\end{small}
where $y_i$ is the category of node $v_i^t$, and $\varepsilon \in (0, 1)$ is the weight.
The low-rank adaptation network coupled with tailor-designed graph contrastive learning incorporates structural information from the target graph, maximizing the mutual information between the pre-trained GNN and the newly added GNN. Therefore, such a strategy mitigates structural discrepancies across graphs, facilitating the adaptation of pre-trained GNNs to target graphs.

We further provide theoretical justification for the robust representation capability of pre-trained GNNs with low-rank adaptation.

\begin{theoremBrief}
\label{theorem_B1}
Let $\overline{g}$ be a target GNN with $\overline{L}$ layers and $g_0$ be an arbitrary frozen GNN with $L$ layers, where $\overline{L}\leqslant L$. Under mild conditions on ranks and network architectures, there exist low-rank adaptations such that the low-rank adapted model $g_0$ becomes exactly equal to $\overline{g}$.
\end{theoremBrief}

The proof of Theorem~\ref{theorem_B1}, along with additional theoretical analysis, are provided in~\appendixname~\ref{appendix:theory}. The theorem suggests that effective cross-graph transfer, i.e., $g^\star$ achieves an optimal solution, can be accomplished through low-rank adaptation applied to the pre-trained GNN $g_0$, thereby equating $g_0$ with $g^\star$.

\subsection{Structure-aware Regularization}

In real-world scenarios, the homophily phenomenon is prevalent in graph data, such as citation networks or co-purchase networks~\cite{DBLP:conf/kdd/LiWXL23a}. 
In general, homophily reflects the tendency for "like to attract like"~\cite{doi:10.1146/annurev.soc.27.1.415}, indicating that connected nodes are prone to sharing similar labels~\cite{DBLP:conf/kdd/LiWXL23a, DBLP:journals/corr/abs-1811-05868}.
In the cross-graph transfer learning context, we leverage the homophily principle to alleviate label scarcity in the target graph. 

Inspired by previous work~\cite{DBLP:journals/corr/abs-2112-00955}, we propose a structure-aware regularization objective based on the homophily principle of graph data. Specifically, we assume that the predicted label vectors of connected neighbors on the target graph are similar, while those of disconnected neighbors are dissimilar. In contrast to GraphSage~\cite{DBLP:conf/nips/HamiltonYL17}, we utilize direct connected neighbors instead of random walk to better satisfy the assumption, which can be formulated as: 
\begin{equation}
\begin{split}
    \mathcal{L}_{str}=&\sum_{i\ne j}[{\boldsymbol{A}_{i,j}^{t}\log \varsigma \left( sim\left( \hat{\boldsymbol{y}}_i,\hat{\boldsymbol{y}}_j \right) \right)} \\
    + &\left( 1-\boldsymbol{A}_{i,j}^{t} \right) \log \left( 1-\varsigma \left( sim\left( \hat{\boldsymbol{y}}_i,\hat{\boldsymbol{y}}_j \right) \right) \right)],
\end{split}
\end{equation}
where $\hat{\boldsymbol{y}}_i$ represents the predicted label vector of node $v_i^{t}$, $sim(\cdot, \cdot)$ represents the inner product, and $\varsigma(\cdot)$ represents the sigmoid function. 
Despite the limited availability of labeled data, the above regularization objective effectively utilizes the inherent homophily property in the graph to mitigate the challenge of label scarcity.

\begin{table*}[!t]
\caption{Comparison of experimental results in public and 10-shot settings. The notations "-PM," "-CS," "-C," "-P," and "-Com" represent the pre-training datasets PubMed, CiteSeer, Cora, Photo, and Computer, respectively. The best experimental results are highlighted in bold, while the second-best results are underscored with a underline.}
\label{tab:compare}
\resizebox{\textwidth}{!}{
\begin{tabular}{@{}cc|cccccccccc@{}}
\toprule
\multicolumn{2}{c|}{\multirow{2}{*}{Method}} & \multicolumn{2}{c}{PubMed} & \multicolumn{2}{c}{CiteSeer} & \multicolumn{2}{c}{Cora} & \multicolumn{2}{c}{Photo} & \multicolumn{2}{c}{Computer} \\ \cmidrule(l){3-4} \cmidrule(l){5-6} \cmidrule(l){7-8} \cmidrule(l){9-10} \cmidrule(l){11-12}
\multicolumn{2}{c|}{} & public & 10-shot & public & 10-shot & public & 10-shot & public & 10-shot & public & 10-shot \\ \midrule
\multicolumn{1}{c|}{\multirow{10}{*}{non-transfer}} & GCN                     & 78.66±0.56          & 73.28±0.93          & 70.50±0.75          & 64.52±1.40         & 82.00±0.97          & 71.88±1.26          & 92.17±0.75          & 85.39±1.57          & 86.66±1.20          & 71.97±1.16          \\
\multicolumn{1}{c|}{} & GAT                     & 78.30±0.43          & 74.96±0.67          & 70.94±1.08          & 64.14±2.12         & 80.58±1.50          & 72.04±0.61          & 92.91±0.22          & 86.38±0.65          & 86.80±0.71          & 74.82±2.36          \\
\multicolumn{1}{c|}{} & GRACE                   & 79.52±0.16          & 75.86±0.11          & 70.34±0.21          & 67.70±0.00         & 82.28±0.04          & 76.40±0.00          & 92.32±0.31          & 86.16±0.02          & 85.54±0.30          & 74.39±0.03          \\
\multicolumn{1}{c|}{} & COSTA                   & 79.94±1.16          & 76.98±1.29          & 70.36±1.29          & 65.56±1.94         & 81.84±0.92          & 76.28±1.42          & 92.04±0.58          & 83.02±0.43          & 87.00±0.33          & 71.28±0.82          \\
\multicolumn{1}{c|}{} & CCA-SSG                 & 80.58±0.85          & \uline{78.76±1.62}    & 73.76±0.75          & 67.46±0.92         & \textbf{83.94±1.02} & 76.84±0.80          & 92.74±0.33          & 85.56±0.67          & 88.08±0.35          & \uline{76.94±1.37}    \\
\multicolumn{1}{c|}{} & HomoGCL                 & \textbf{81.04±0.05} & \textbf{79.42±0.04} & 71.38±0.04          & 65.40±0.00         & 82.40±0.00          & 75.70±0.00          & 92.43±0.08          & 84.30±0.35          & 87.75±0.29          & 76.85±0.41          \\
\multicolumn{1}{c|}{} & GPPT                    & 77.78±0.31          & 74.84±0.55          & 67.56±0.33          & 64.14±0.63         & 80.16±0.38          & 72.94±0.24          & 92.10±0.23          & 86.32±0.56          & \textbf{88.34±0.18} & \textbf{77.30±0.44} \\ 
\multicolumn{1}{c|}{} & GPF         & 79.48±0.51 & 74.78±1.39 & 71.32±0.26 & 66.76±0.68 & 82.10±0.26 & 76.30±0.51 & 91.61±0.60 & 86.76±1.55 & 77.44±1.37 & 70.48±1.96 \\
\multicolumn{1}{c|}{} & GraphPrompt & 75.23±0.93 & 74.27±1.44 & 69.71±1.06 & 65.88±0.91 & 79.90±0.74 & 75.02±0.56 & 86.35±0.41 & 84.06±0.89 & 72.43±0.27 & 66.54±0.60 \\
\multicolumn{1}{c|}{} & ProG        & 75.85±0.45 & 71.43±0.98 & 71.31±0.99 & 68.48±1.26 & 82.03±0.59 & 76.69±0.83 & 85.18±1.70 & 86.86±0.62 & 66.65±1.95 & 67.20±1.54 \\ \midrule
\multicolumn{1}{c|}{\multirow{25}{*}{transfer}} & GRACE$_t$-PM                & 79.44±0.15          & 75.80±0.16          & 67.74±0.05          & 57.40±0.00         & 76.82±0.24          & 64.44±0.30          & 92.14±0.14          & 85.95±0.10          & 84.81±0.41          & 76.15±0.09          \\
\multicolumn{1}{c|}{} & GRACE$_t$-CS                & 76.58±0.04          & 72.46±0.39          & 70.50±0.24          & 67.70±0.00         & 79.04±0.05          & 71.70±0.00          & 92.46±0.38          & 86.81±0.02          & 85.08±0.62          & 76.05±0.01          \\
\multicolumn{1}{c|}{} & GRACE$_t$-C                 & 73.00±0.00          & 66.78±0.04          & 67.10±0.12          & 58.14±0.13         & 82.32±0.04          & 76.40±0.00          & 92.23±0.21          & 86.12±0.03          & 84.46±0.24          & 75.33±0.30          \\
\multicolumn{1}{c|}{} & GRACE$_t$-P                 & 71.10±0.00          & 57.70±0.62          & 58.00±0.00          & 49.38±0.04         & 72.40±0.00          & 57.88±0.04          & 92.25±0.34          & 86.20±0.02          & 84.24±0.21          & 74.57±0.01          \\
\multicolumn{1}{c|}{} & GRACE$_t$-Com               & 70.42±0.04          & 64.12±0.08          & 61.20±0.00          & 57.90±0.00         & 67.46±0.05          & 55.50±0.00          & 92.25±0.43          & 85.31±0.00          & 85.89±0.50          & 74.38±0.04          \\ \cmidrule{2-12}
\multicolumn{1}{c|}{} & GTOT-PM  & 76.48±1.12   & 71.92±0.66  & 69.96±1.31    & 60.10±0.54   & 78.82±1.07  & 68.54±0.44 & 90.18±0.88  & 83.69±0.52  & 84.88±0.31    & 64.60±3.28    \\
\multicolumn{1}{c|}{} & GTOT-CS  & 75.76±0.73   & 70.74±0.99  & 71.30±1.35    & 60.90±0.90   & 79.36±1.02  & 69.90±1.07 & 90.75±0.59  & 83.19±0.64  & 84.97±0.38    & 66.59±0.99    \\
\multicolumn{1}{c|}{} & GTOT-C   & 75.66±0.80   & 70.72±0.30  & 68.98±1.01    & 60.98±0.89   & 79.36±0.78  & 69.84±0.86 & 90.24±1.33  & 84.21±0.86  & 85.30±0.11    & 66.46±0.59    \\
\multicolumn{1}{c|}{} & GTOT-P   & 76.44±0.63   & 71.42±0.70  & 69.28±1.09    & 60.86±0.80   & 79.40±2.26  & 69.04±0.75 & 90.42±0.42  & 83.37±0.97  & 84.66±0.42    & 64.60±1.19    \\
\multicolumn{1}{c|}{} & GTOT-Com & 74.24±0.43   & 70.70±0.46  & 68.56±0.55    & 61.14±1.45   & 79.64±1.00  & 69.86±0.85 & 90.43±0.42  & 83.82±0.82  & 84.88±0.44    & 67.04±0.83    \\ \cmidrule{2-12}
\multicolumn{1}{c|}{} & AdapterGNN-PM  & 76.44±0.97   & 72.78±0.72  & 62.76±1.42    & 58.64±0.44   & 75.54±1.46  & 63.82±1.44 & 92.39±0.32  & 88.24±0.39  & 88.00±0.18    & 75.54±0.65    \\
\multicolumn{1}{c|}{} & AdapterGNN-CS  & 74.12±1.72   & 64.92±0.45  & 66.38±0.49    & 66.68±0.41   & 77.82±0.44  & 70.34±1.63 & 92.89±0.18  & 87.56±0.17  & 87.96±0.23    & 74.17±1.34    \\
\multicolumn{1}{c|}{} & AdapterGNN-C   & 73.86±0.11   & 60.76±1.90  & 64.22±0.58    & 60.94±0.37   & 82.08±0.37  & 72.62±2.58 & 92.77±0.42  & 87.07±0.19  & 87.91±0.17    & 74.66±1.09    \\
\multicolumn{1}{c|}{} & AdapterGNN-P   & 72.94±0.42   & 63.44±0.84  & 64.20±0.43    & 53.02±0.86   & 75.50±1.33  & 57.64±2.76 & 92.58±0.50  & 88.18±0.29  & 87.62±0.38    & 75.00±0.52    \\
\multicolumn{1}{c|}{} & AdapterGNN-Com & 72.50±0.62   & 58.94±2.89  & 63.64±0.68    & 58.74±0.33   & 74.12±0.73  & 56.42±2.50 & 92.66±0.45  & 88.20±0.77  & 87.63±0.49    & 72.64±1.96    \\ \cmidrule{2-12}
\multicolumn{1}{c|}{} & GraphControl-PM & 78.30±0.43 & 75.96±1.00           & 69.02±1.65 & 60.82±0.41           & 77.84±0.67 & 69.32±2.11           & 90.73±0.75 & 86.65±0.51           & 85.94±0.96 & 74.47±2.43           \\
\multicolumn{1}{c|}{} & GraphControl-CS  & 75.98±0.66 & 72.56±0.63           & 70.80±0.97 & 68.56±0.98           & 77.54±1.24 & 74.04±0.79           & 90.15±0.67 & 87.44±0.29           & 86.36±0.26 & 74.03±1.01           \\
\multicolumn{1}{c|}{} & GraphControl-C  & 74.52±0.88 & 66.00±0.66           & 66.20±0.94 & 58.70±0.56           & 77.14±1.72 & 76.44±0.31           & 90.52±0.48 & 86.57±0.70           & 85.99±0.51 & 73.17±1.50           \\
\multicolumn{1}{c|}{} & GraphControl-P & 74.58±1.99 & 58.94±0.69           & 59.12±1.34 & 53.36±1.92           & 73.46±1.73 & 65.72±0.44           & 90.67±0.50 & 86.23±0.59           & 86.11±0.50 & 71.86±2.38           \\
\multicolumn{1}{c|}{} & GraphControl-Com  & 72.90±0.31 & 65.60±0.42           & 60.54±0.90 & 60.68±0.51           & 73.82±1.50 & 63.32±0.62           & 90.73±0.73 & 83.20±0.36           & 86.08±0.62 & 69.18±0.66 \\ \cmidrule{2-12}
\multicolumn{1}{c|}{} & GraphLoRA-PM            & \uline{80.86±0.39}    & 78.06±0.59          & \textbf{74.20±0.47} & \uline{74.62±0.57}   & \uline{82.42±0.40}    & \uline{78.08±0.3}     & \uline{93.00±0.36}    & 88.34±0.51          & 87.70±0.63          & 76.54±0.39          \\
\multicolumn{1}{c|}{} & GraphLoRA-CS            & 80.64±0.43          & 78.08±0.67          & \uline{74.08±0.26}    & \textbf{74.80±0.6} & 82.00±0.23          & \textbf{78.30±0.46} & \textbf{93.08±0.11} & \uline{89.00±0.43}    & 87.72±0.45          & 76.44±0.05          \\
\multicolumn{1}{c|}{} & GraphLoRA-C             & 80.38±0.50          & 77.78±0.55          & 73.98±0.45          & \uline{74.62±0.65}   & 82.00±0.80          & 78.00±0.51          & 92.92±0.29          & 88.69±0.35          & \uline{88.10±0.31}    & 76.45±0.48          \\
\multicolumn{1}{c|}{} & GraphLoRA-P             & 78.46±0.67          & 74.84±1.67          & 72.80±0.58          & 72.02±1.64         & 81.28±0.40          & 76.54±0.54          & 92.47±0.33          & 88.89±0.89          & 87.35±0.49          & 75.27±0.22          \\
\multicolumn{1}{c|}{} & GraphLoRA-Com           & 79.02±0.77          & 74.44±1.36          & 72.72±0.35          & 72.12±1.29         & 80.84±1.05          & 76.44±0.9           & 92.42±0.22          & \textbf{89.07±0.22} & 87.28±0.31          & 75.35±1.26          \\ \bottomrule
\end{tabular}}
\end{table*}

\subsection{Optimization}
Finally, we employ the following output layer to classify the target nodes based on the output of the GNN, 
\begin{align}
    \tilde{y}_i=&c\left( \boldsymbol{h}_i+\boldsymbol{h}_{i}^{'} \right), 
\end{align}
where $c(\cdot)$ represents the classifier, and $\tilde{y}_i=argmax \left( \hat{\boldsymbol{y}}_i \right) $ denotes the predicted class of node $v_i^t$. The classification loss function is defined as follows:
\begin{align}
    \mathcal{L} _{cls}=&-\frac{1}{N^t}\sum_i{y_i\log \tilde{y}_i+\left( 1-y_i \right) \log \left( 1-\tilde{y}_i \right)}.
\end{align}

After acquiring the pre-trained GNN from the source graph, we proceed to fine-tune it by utilizing the labeled data available on the target graph. To achieve this, we employ multitask learning to jointly optimize multiple objective functions. The overall objective function is defined as follows:
\begin{align}
    \mathcal{L} =\mathcal{L} _{cls}+\lambda _1\mathcal{L} _{smmd}+\lambda _2\mathcal{L} _{cl}+\lambda _3\mathcal{L} _{str}+\lambda _4\left\| \varTheta \right\|,
\end{align}
where the last term acts as a regularization term to prevent overfitting, and the weight coefficients $\lambda _{1-4}$ determine the importance of each objective function in the overall optimization process.

\subsection{Complexity Analysis}
In this section, we analyze the time complexity of GraphLoRA. For a target graph with $N^t$ nodes and $M$ edges, the node feature adaptation module performs feature mapping with a runtime of $\mathcal{O}(N^t)$. By leveraging fast approximations~\cite{DBLP:conf/focs/AndersenCL06, DBLP:conf/sigmod/WeiHX0SW18}, the diffusion matrix $\boldsymbol{S}$ can be obtained in $\mathcal{O}(N^t)$. $\mathcal{L}_{smmd}$ necessitates calculating distances between node pairs in each batch, achievable in $\mathcal{O}(N^tb)$ through the utilization of sampling techniques, where $b$ denotes the batch size. Similarly, in the structural knowledge transfer learning module, $\mathcal{L}_{cl}$ requires calculating similarity between node pairs in each batch, also with a complexity of $\mathcal{O}(N^tb)$. As for the structure-aware regularization objective, $\mathcal{L}_{str}$ considers connected nodes as positive samples and samples a small subset of nodes as negative samples, resulting in a complexity of $\mathcal{O}(M)$. In summary, the fine-tuning time complexity of GraphLoRA is $\mathcal{O}(N^tb + M)$, which is lightweight considering that the batch size is typically small.

\section{Experiments}
In this section, we conduct extensive experiments on benchmark datasets to evaluate GraphLoRA's effectiveness in cross-graph transfer learning, aiming to answer the following research questions:
\begin{itemize}
  \item[\textbf{RQ1}:] How effective and efficient is GraphLoRA?
  \item[\textbf{RQ2}:] Is GraphLoRA sensitive to hyperparameters?
  \item[\textbf{RQ3}:] How do different modules contribute to its effectiveness?
  \item[\textbf{RQ4}:] Can GraphLoRA mitigate catastrophic forgetting?
  \item[\textbf{RQ5}:] Can GraphLoRA learn more distinguishable representations?
  % Can GraphLoRA learn more distinguishable node representations than other baselines?
\end{itemize}

\begin{table}[!t]
\caption{Statistics of datasets.}
\label{tab:stat}
\begin{tabular}{@{}c|cccc@{}}
\toprule
Dataset  & \#Nodes & \#Edges & \#Features & \#Classes \\ \midrule
PubMed   & 19,717   & 88,651   & 500        & 3         \\
CiteSeer & 3,327    & 9,228    & 3,703       & 6         \\
Cora     & 2,708    & 10,556   & 1,433       & 7         \\
Photo    & 7,650    & 238,163  & 745        & 8         \\
Computer & 13,752   & 491,722  & 767        & 10        \\
Reddit   & 232,965 & 114,615,892  & 602   & 41        \\
ogbn-arxiv   & 169,343 & 1,166,243  & 128   & 40        \\
ogbn-products   & 2,449,029 &  61,859,140  & 100   & 47      \\ \bottomrule
\end{tabular}
\end{table}

\subsection{Experimental Setup}
\subsubsection{Datasets}
We evaluate GraphLoRA on eight datasets: PubMed, CiteSeer, Cora~\cite{DBLP:conf/icml/YangCS16}, and ogbn-arxiv~\cite{DBLP:conf/nips/HuFZDRLCL20} are citation networks, where each node represents a paper, edges denote citations, and the node labels indicate the topics of the papers. Photo, Computer~\cite{DBLP:journals/corr/abs-1811-05868}, and ogbn-products~\cite{DBLP:conf/nips/HuFZDRLCL20} are Amazon product co-purchasing networks, where each node represents a product, edges represent co-purchases, and labels denote the product categories. In the Reddit~\cite{DBLP:conf/nips/HamiltonYL17} dataset, nodes represent posts, edges indicate posts commented on by the same user, and labels represent the communities of the posts. Statistics for these datasets are presented in~\tablename~\ref{tab:stat}. Detailed descriptions of these datasets are provided in~\appendixname~\ref{appendix:data}.

\subsubsection{Baselines}
Baselines include supervised methods (GCN~\cite{DBLP:conf/iclr/KipfW17} and GAT~\cite{DBLP:conf/iclr/VelickovicCCRLB18}), graph contrastive learning methods (GRACE~\cite{DBLP:journals/corr/abs-2006-04131}, COSTA~\cite{DBLP:conf/kdd/ZhangZSKK22}, CCA-SSG~\cite{DBLP:conf/nips/ZhangWYWY21}, and HomoGCL~\cite{DBLP:conf/kdd/LiWXL23a}), graph prompt learning methods (GPPT~\cite{DBLP:conf/kdd/SunZHWW22}, GPF~\cite{DBLP:conf/nips/FangZYWC23}, GraphPrompt~\cite{DBLP:conf/www/LiuY0023}, and ProG~\cite{DBLP:conf/kdd/SunCLLG23}), and transfer learning methods (GRACE$_t$, % \rev{
GTOT~\cite{DBLP:conf/ijcai/ZhangXHRB22}, AdapterGNN~\cite{DBLP:conf/aaai/LiH024}
% } 
and GraphControl~\cite{DBLP:journals/corr/abs-2310-07365}). 
Among these, GRACE$_t$ involves pretraining a GNN on the source graph using GRACE and then transferring it to the target graph for testing.
Detailed descriptions of baselines are provided in~\appendixname~\ref{appendix:baseline}.

\subsubsection{Settings}
For GraphLoRA, we use a two-layer GAT model as the backbone. The projector $p\left( \cdot; \boldsymbol{\omega} \right)$ and classifier $c(\cdot)$ are implemented with a single linear layer. The GNN is pre-trained using GRACE and fine-tuned on the target graph with our method. Experiments are conducted in the public setting with sufficient labels, and in the 5-shot and 10-shot settings with limited labels. In the public setting, PubMed, CiteSeer, and Cora are split using public partitions~\cite{DBLP:conf/icml/YangCS16}, where each category has 20 training labels. For Photo, Computer, and Reddit, we randomly split the datasets into training (10\%), validation (10\%), and testing (80\%) sets. For the ogbn datasets, we use the public splits provided by the authors~\cite{DBLP:conf/nips/HuFZDRLCL20}. In the 5-shot and 10-shot settings, each category in the training set contains only 5 and 10 labels, respectively, with 80\% for testing and the remaining data for validation. For all methods, we conduct the experiments five times and report the average accuracy and standard deviation. The additional results for the 5-shot setting are provided in the Appendix due to space constraints.

For GCN and GAT, we train the GNN using labeled data from the target graph. For graph contrastive learning and graph prompt learning methods, we pre-train the GNN unsupervised on the target graph, then freeze the model and fine-tune either a linear classifier or a graph prompt using the target labels. For transfer learning methods, we pre-train the GNN unsupervised on the source graph, then transfer the model to the target graph and fine-tune a linear classifier or adapter using the target labels. For all methods, the GNN’s hidden dimensions are fixed at 512 and 256. The learning rate and weight decay are tuned within [1e-5, 1e-1]. For GraphLoRA, we set $r=32$, and $\lambda$ is tuned within [0.1, 10]. The Adam~\cite{DBLP:journals/corr/KingmaB14} optimizer is used for optimization, and other hyperparameters for baselines are tuned as suggested by the authors.

\subsection{Performance Comparison (RQ 1)}
The performance of GraphLoRA on node classification tasks is presented in~\tablename~\ref{tab:compare}. GraphLoRA achieves either the best or second-best performance in most cases, underscoring its effectiveness. Compared to non-transfer learning scenarios, transfer learning scenarios are more challenging. Nevertheless, GraphLoRA achieves an average improvement of 1.01\% over the best baseline results and 3.33\% over GRACE. Specifically, it achieves an average improvement of 2.23\% over GRACE in the public setting and 4.43\% in the 10-shot setting. GraphLoRA shows a more significant performance improvement in the 10-shot setting, underscoring its effectiveness in scenarios with scarce labels.

\subsubsection{Cross-graph Transfer Learning}
From~\tablename~\ref{tab:compare}, we can observe that transfer learning methods exhibit poorer performance compared to non-transfer learning methods, highlighting the significant challenge of cross-graph transfer. In contrast, GraphLoRA demonstrates impressive transfer learning capabilities, even in cross-domain scenarios. Specifically, GraphLoRA achieves an average improvement of 10.12\% over GRACE$_t$, indicating that direct transfer of pre-trained GNNs results in suboptimal performance. Additionally, GraphLoRA achieves average improvements of 8.21\% over GTOT, 9.78\% over AdapterGNN, and 8.74\% over GraphControl.

\begin{figure}[!t]
  \centering
  \begin{subfigure}{0.23\textwidth}
    \includegraphics[width=\linewidth]{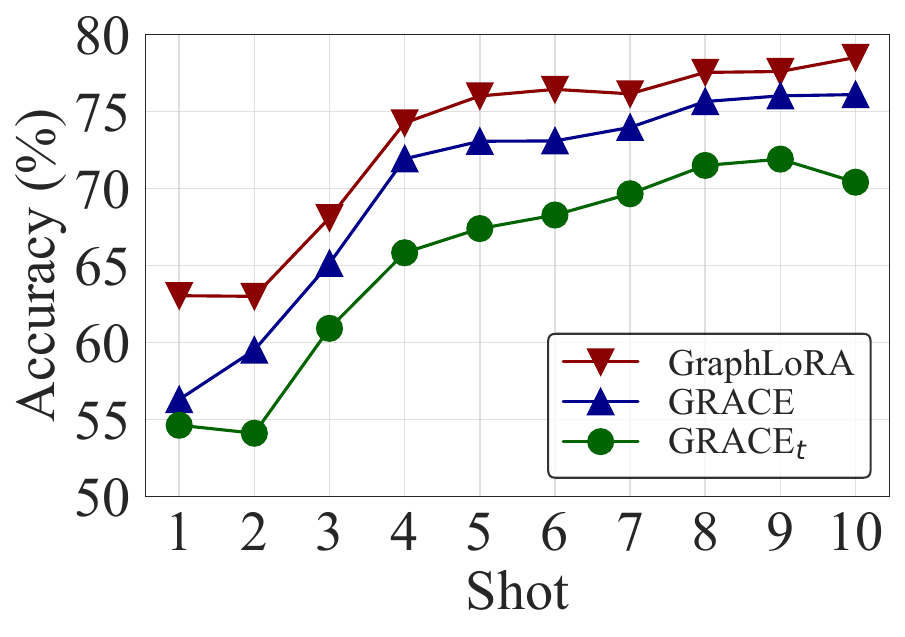}
    \caption{Average accuracy across different shots.}
    \label{fig:shots}
  \end{subfigure}
  \begin{subfigure}{0.23\textwidth}
    \includegraphics[width=\linewidth]{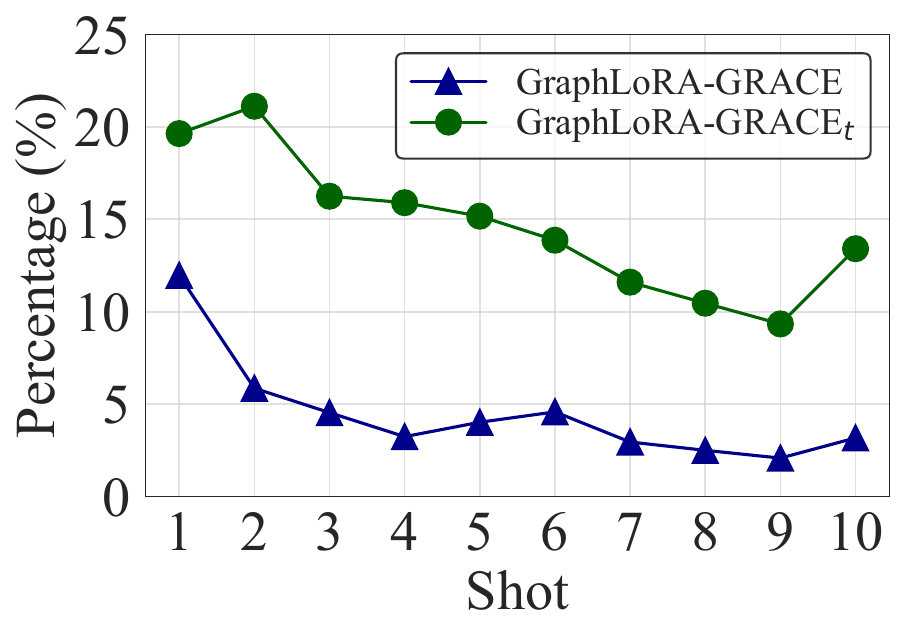}
    \caption{GraphLoRA improvements over GRACE and GRACE$_t$.}
    \label{fig:improvement}
  \end{subfigure}
  \caption{Experimental results across different shots.}
  \label{fig:impact_of_shots}
\end{figure}

\begin{table}[]
\setlength{\tabcolsep}{1mm}
\caption{Comparison of runtimes of different methods in public and 10-shot settings. For transfer learning methods, we report the fine-tuning runtimes.}
\label{tab:runtime}
\resizebox{\linewidth}{!}{
\begin{tabular}{c|cccccccccc|c}
\hline
\multirow{2}{*}{Method} & \multicolumn{2}{c}{PubMed} & \multicolumn{2}{c}{CiteSeer} & \multicolumn{2}{c}{Cora} & \multicolumn{2}{c}{Photo} & \multicolumn{2}{c|}{Computer} & \multirow{2}{*}{Avg.} \\ \cmidrule(l){2-3} \cmidrule(l){4-5} \cmidrule(l){6-7} \cmidrule(l){8-9} \cmidrule(l){10-11}
& public       & 10-shot     & public       & 10-shot       & public      & 10-shot    & public      & 10-shot     & public         & 10-shot       &                       \\ \hline
GCN & 2.1s & 1.7s & 2.3s & 2.3s & 2.1s & 2.1s & 11.6s & 11.2s & 12.6s & 12.1s & 6.0s \\
GAT & 2.8s & 2.8s & 2.8s & 9.0s & 2.5s & 8.0s & 14.3s & 5.5s & 16.1s & 4.4s & 6.8s \\
GRACE & 88.5s & 88.7s & 6.4s & 19.2s & 16.9s & 18.8s & 31.7s & 24.8s & 64.6s & 62.2s & 42.2s \\ 
COSTA & 1074.5s & 524.0s & 61.7s & 127.4s & 288.1s & 135.2s & 109.2s & 71.8s & 28.0s & 77.1s & 249.7s \\
CCA-SSG & 8.7s & 5.9s & 6.8s & 7.8s & 4.5s & 5.0s & 4.9s & 6.9s & 13.4s & 9.6s & 7.3s \\
HomoGCL & 166.9s & 155.2s & 6.6s & 7.1s & 5.1s & 5.3s & 26.8s & 17.9s & 100.8s & 92.9s & 58.5s \\
GPPT & 23.6s & 7.9s & 36.5s & 32.8s & 37.5s & 27.8s & 181.7s & 96.5s & 378.7s & 179.7s & 100.3s \\
GPF         & 26.4s  & 38.6s   & 2.6s     & 4.3s    & 7.8s   & 8.4s    & 185.9s & 90.0s   & 246.0s    & 109.3s  & 71.9s  \\
GraphPrompt & 16.2s  & 32.4s   & 2.4s     & 2.4s    & 6.6s   & 5.7s    & 21.9s  & 9.7s    & 26.7s     & 31.5s   & 15.55s \\
ProG        & 24.5s  & 15.6s   & 6.0s     & 7.5s    & 11.3s  & 21.7s   & 45.8s  & 23.5s   & 43.6s     & 33.3s   & 23.3s  \\ \hline
GRACE$_t$ & 0.6s & 0.4s & 11.5s & 12.0s & 0.6s & 12.0s & 12.1s & 4.1s & 12.9s & 5.8s & 7.2s \\
GTOT & 17.8s & 79.3s & 9.8s & 17.6s & 6.2s & 10.9s & 23.2s & 19.9s & 79.3s & 56.3s & 32.0s \\
AadpterGNN & 18.9s & 7.8s & 14.6s & 17.7s & 11.6s & 17.1s & 45.3s & 38.6s & 106.9s & 65.4s & 34.4s \\
GraphControl & 0.9s & 1.1s & 0.3s & 1.5s & 1.0s & 0.8s & 15.6s & 9.8s & 29.5s & 13.0s & 7.4s \\
GraphLoRA & 43.7s & 11.2s & 5.0s & 8.8s & 10.7s & 3.6s & 44.6s & 17.5s & 108.9s & 56.3s & 31.0s \\ \hline
\end{tabular}}
\end{table}

\subsubsection{Scarce Labeling Impact on Performance}
To further explore the impact of label scarcity on performance, we investigate the performance of GraphLoRA across the 1-shot to 10-shot setting, as illustrated in~\figurename~\ref{fig:impact_of_shots}. The figure reveals that, overall, GraphLoRA demonstrates a greater performance improvement compared to GRACE and GRACE$_t$ in scenarios with scarce labels. This observation not only reaffirms our earlier analysis but also substantiates the crucial role of transfer learning in scenarios with scarce labels. Furthermore, it is noteworthy that GraphLoRA consistently exhibits a more substantial performance improvement compared to GRACE$_t$, providing additional confirmation that the direct transfer of pre-trained GNNs will result in suboptimal performance.

\subsection{Efficiency Comparison (RQ 1)}
Efficiency is a critical consideration in practical applications~\cite{DBLP:conf/kdd/LiuGLLXCWYZDDX21}. To evaluate the efficiency of GraphLoRA, we measure the runtime of different methods in both public and 10-shot settings on the same device, as depicted in~\tablename~\ref{tab:runtime}. For transfer learning methods, we present the total runtime until model convergence during fine-tuning, while for other methods, we present the total runtime until model convergence during training. From~\tablename~\ref{tab:runtime}, it is shown that the average runtime of GraphLoRA is lower than that of most baselines, indicating its high efficiency. It is noteworthy that GraphLoRA exhibits higher efficiency in the 10-shot setting compared to other baselines. This may be attributed to the effective mitigation of label sparsity through the structure-aware regularization objective, thereby facilitating easier model convergence.

\begin{table}[]
\setlength{\tabcolsep}{1mm}
\caption{Experimental results on the Reddit dataset.}
\label{tab:large_scale}
\resizebox{\linewidth}{!}{
\begin{tabular}{c|llllll}
\hline
Method       & \multicolumn{1}{c}{PM→R} & \multicolumn{1}{c}{CS→R} & \multicolumn{1}{c}{C→R} & \multicolumn{1}{c}{P→R} & \multicolumn{1}{c}{C→R} & \multicolumn{1}{c}{R→R} \\ \hline
GTOT         & 93.04±0.15               & 92.99±0.11               & 93.15±0.10              & 93.11±0.08              & 93.10±0.12              & 93.18±0.07              \\
AdapterGNN   & 91.21±0.10               & 91.61±0.07               & 91.30±0.06              & 91.07±0.23              & 91.18±0.08              & \textbf{93.89±0.12}              \\
GraphControl & 92.79±0.12               & 93.01±0.10               & 92.93±0.10              & 92.76±0.11              & 92.67±0.13              & 93.14±0.11              \\
GraphLoRA    & \textbf{93.25±0.07}               & \textbf{93.22±0.10}               & \textbf{93.44±0.09}              & \textbf{93.48±0.10}              & \textbf{93.44±0.08}              & 93.58±0.07              \\ \hline
\end{tabular}}
\end{table}

\begin{table}[]
\caption{Performance and runtime on large-scale datasets, where OOM indicates an "out-of-memory" issue.}
\label{tab:large_graphs}
\setlength{\tabcolsep}{1mm}
\resizebox{\linewidth}{!}{
\begin{tabular}{c|cccccc}
\hline
Method      & \multicolumn{2}{c}{Reddit} & \multicolumn{2}{c}{ogbn-arxiv} & \multicolumn{2}{c}{ogbn-products} \\ \hline
GRACE       & 92.86±0.02    & \textbf{301.7s}     & 67.65±0.11       & 178.2s      & 73.62±0.31        & 2296.1s       \\
COSTA       & OOM           & OOM        & OOM              & OOM         & OOM               & OOM           \\
CCA-SSG     & 78.76±0.16    & 580.0s     & 67.76±0.18       & \textbf{84.8s}       & 66.38±0.49        & 1533.3s       \\
HomoGCL     & OOM           & OOM        & OOM              & OOM         & OOM               & OOM           \\
GPPT        & 92.03±0.04    & 4293.1s    & 65.82±0.23       & 593.5s      & 67.93±0.27    & 22642.5s               \\
GPF         & 92.10±0.07    & 831.7s     & 67.11±0.17       & 150.0s      & 74.04±0.50        & \textbf{1283.0s}       \\
GraphPrompt & 90.16±0.03    & 983.0s     & 57.62±0.08       & 351.0s      & OOM               & OOM           \\
ProG        & 92.29±0.05    & 633.1s     & 67.90±0.15       & 140.9s      & OOM               & OOM           \\ \hline
GraphLoRA   & \textbf{93.58±0.07}    & 785.2s     & \textbf{68.61±0.20}       & 192.7s      & \textbf{75.05±0.12}        & 4077.6s       \\ \hline
\end{tabular}}
\end{table}

\begin{table*}[]
\caption{Ablation experiment results in public and 10-shot settings. The best experimental results are highlighted in bold.}
\label{tab:ablation}
\resizebox{\textwidth}{!}{
\begin{tabular}{@{}c|cccccccccc@{}}
\toprule
\multicolumn{1}{c|}{\multirow{2}{*}{Variants}} & \multicolumn{2}{c}{PubMed}                        & \multicolumn{2}{c}{CiteSeer}              & \multicolumn{2}{c}{Cora}                 & \multicolumn{2}{c}{Photo}                 & \multicolumn{2}{c}{Computer}             \\ \cmidrule(l){2-3} \cmidrule(l){4-5} \cmidrule(l){6-7} \cmidrule(l){8-9} \cmidrule(l){10-11}
\multicolumn{1}{l|}{}                          & public              & 10-shot & public              & 10-shot             & public              & 10-shot            & public              & 10-shot             & public              & 10-shot             \\ \midrule
w/ mmd                                       & 79.46±0.82          & 76.44±1.76                  & 73.12±0.43          & 74.08±0.91          & 81.46±0.54          & 76.80±0.86         & 92.70±0.37          & 87.91±1.03          & 87.75±0.12          & 76.01±0.24          \\
w/o smmd                                       & 77.58±0.19          & 75.42±0.58                  & 71.52±0.72          & 69.24±2.71          & 81.20±0.30          & 76.38±0.76         & 92.56±0.33          & 87.90±1.32          & 87.59±0.17          & 74.96±0.68          \\
w/o cl                                        & 79.76±0.47          & 77.70±0.94                  & 73.70±0.66          & 74.42±0.56          & 81.40±0.70          & 77.24±0.60         & 92.64±0.14          & 87.81±0.36          & 87.67±0.44          & 76.10±0.23          \\
w/o str                                        & 79.88±0.49          & 77.54±0.83                  & 71.12±0.55          & 68.12±0.42          & 80.64±0.65          & 74.50±1.34         & 92.63±0.49          & 87.96±0.83          & 87.60±0.27          & 75.87±0.67          \\
w/o lrd                                       & 80.06±0.90          & 77.20±1.29                  & 72.52±1.65          & 72.94±0.88          & 80.88±0.36          & 77.16±1.87         & 92.63±0.20          & 88.15±0.16          & 87.79±0.34          & 76.24±0.52          \\
w/o nfa                                        & 79.72±0.13          & 76.86±0.09                  & 68.20±0.25          & 27.76±0.65          & 75.24±0.55          & 69.68±0.24         & 87.26±0.82          & 83.17±0.34          & 80.97±0.82          & 69.46±1.06          \\
w/o sktl                                        & 80.02±0.41          & 77.22±0.78                  & 72.84±0.61          & 73.72±0.33          & 81.72±1.20          & 77.08±0.74         & 92.61±0.42          & 88.17±0.97          & \textbf{87.92±0.23} & 76.23±0.48          \\ \midrule
GraphLoRA                                      & \textbf{80.86±0.39} & \textbf{78.06±0.59}         & \textbf{74.20±0.47} & \textbf{74.62±0.57} & \textbf{82.42±0.40} & \textbf{78.08±0.30} & \textbf{93.00±0.36} & \textbf{88.34±0.51} & 87.70±0.63          & \textbf{76.54±0.39} \\ \bottomrule
\end{tabular}}
\end{table*}

\subsection{Results on Large-Scale Dataset (RQ 1)}
GraphLoRA can be easily applied to large-scale graphs using sampling techniques. We conduct experiments in both transfer and non-transfer scenarios. In the transfer scenario, we compare GraphLoRA with other transfer learning methods on the Reddit dataset, with the results shown in~\tablename~\ref{tab:large_scale}. These results show that GraphLoRA outperforms all other methods in most cases. In the non-transfer scenario, we compare GraphLoRA with other non-transfer learning methods on three large-scale graphs. The performance and runtime results are shown in~\tablename~\ref{tab:large_graphs}. GraphLoRA consistently achieves the best performance and demonstrates competitive runtime efficiency, offering a better balance between effectiveness and efficiency. Overall, GraphLoRA demonstrates superior performance, highlighting its effectiveness on large-scale graphs.

\subsection{Hyperparameter Analysis (RQ 2)}
% \rev{
\subsubsection{Impact of $\lambda$}
The model's performance varies with different combinations of coefficients in the objective function. To investigate GraphLoRA's sensitivity to hyperparameters, we conduct a parameter analysis on these coefficients. 
In our experiments, we tune the values of $\lambda_1$, $\lambda_2$, and $\lambda_3$ within the range of [0.1,10]. The experimental results are presented in~\figurename~\ref{fig:lambda_ana}, demonstrating that GraphLoRA's performance remains generally stable, indicating low sensitivity to hyperparameters.
Additionally, the impact of parameter adjustments on performance varies across datasets. For instance, the performance of GraphLoRA improves with an increase in $\lambda_3$ on the CiteSeer dataset. Conversely, its performance decreases with an increase in $\lambda_3$ on the PubMed and Computer datasets. To achieve optimal performance, strategies such as grid search, random search, and Bayesian optimization can be employed to obtain the best hyperparameter combination~\cite{DBLP:journals/corr/abs-2003-05689}.
Overall, GraphLoRA demonstrates low sensitivity to hyperparameters, although the optimal parameter combination varies across different datasets.

\begin{figure}[!t]
  \centerline{\includegraphics[width=1\linewidth]{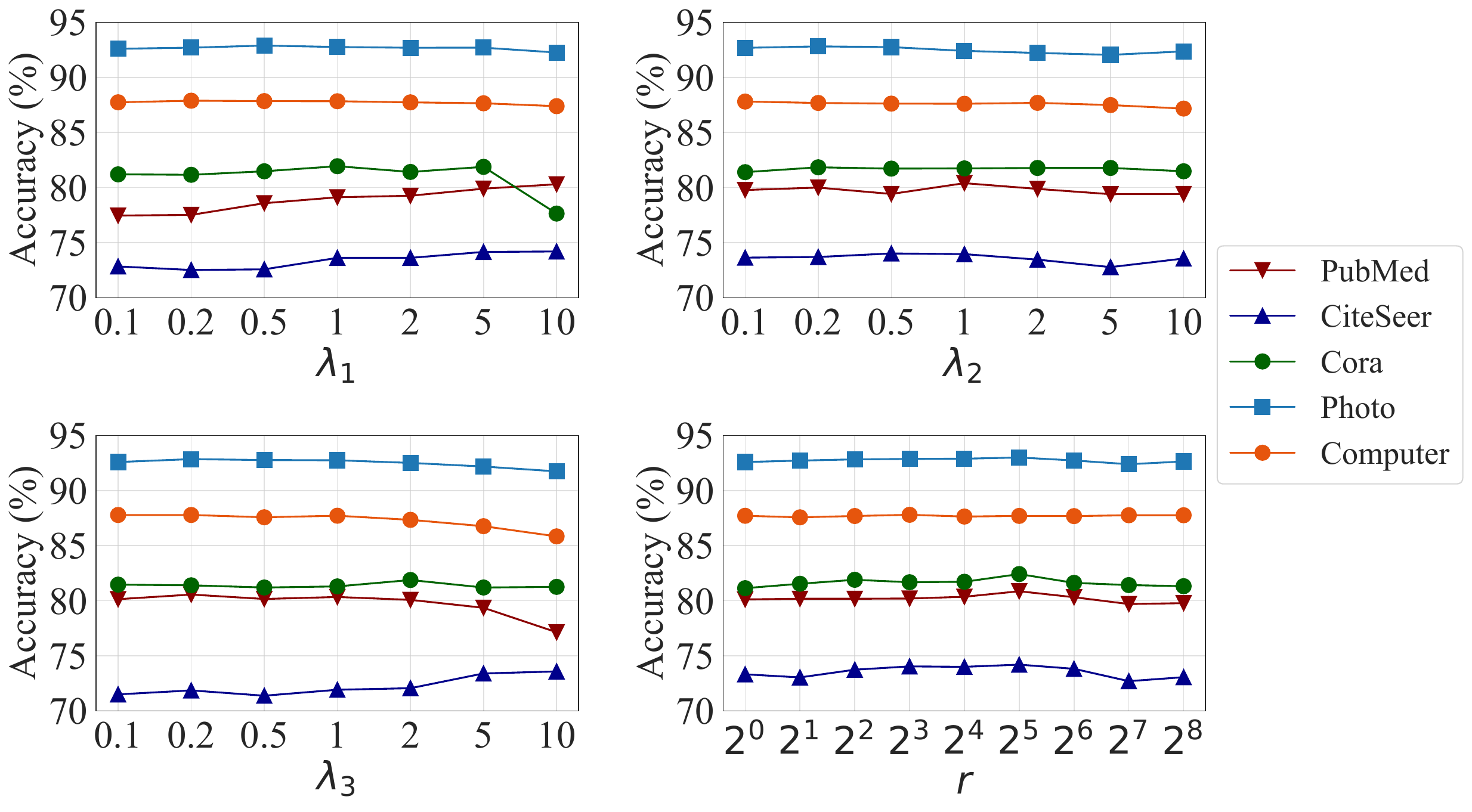}}
  \caption{Performance across varying hyperparameter values.}
  \label{fig:lambda_ana}
\end{figure}

\subsubsection{Impact of $r$}
The hyperparameter $r$ determines the parameter size of GraphLoRA. We evaluate GraphLoRA's performance across $r$ values ranging from $2^0$ to $2^8$ in the public setting. The results are depicted in~\figurename~\ref{fig:lambda_ana}. The figure reveals that GraphLoRA maintains stable performance across different values of $r$. Even when $r$ is set to 1, GraphLoRA exhibits commendable performance, aligning with the understanding that a small $r$ value is adequate for LoRA~\cite{DBLP:conf/iclr/HuSWALWWC22}. Notably, GraphLoRA experiences a decline in performance when $r$ is too small or too large. Generally, optimal performance is achieved when $r$ falls within the range of $2^3$ to $2^5$, with the tunable parameter ranging from 7\% to 20\%. This can be attributed to a small $r$ limiting parameters for effective fine-tuning, whereas a large $r$ may lead to overfitting due to an abundance of tunable parameters.

\subsection{Ablation Studies (RQ 3)}
To evaluate the effectiveness of each module in GraphLoRA, we compare it with seven model variants. Specifically, 
"w/ mmd" represents the method using the target term $\mathcal{L} _{mmd}$ rather than $\mathcal{L} _{smmd}$. 
"w/o smmd", "w/o cl", and "w/o str" represent methods without using the target terms $\mathcal{L} _{smmd}$, $\mathcal{L} _{cl}$, and $\mathcal{L} _{str}$, respectively. Additionally, "w/o lrd" is the method without employing low-rank decomposition for weight updates, while "w/o nfa" and "w/o sktl" represent methods without utilizing the node feature adaptation module and structural knowledge transfer learning module, respectively. 
The results in the public setting and 10-shot setting, following pre-training on the PubMed dataset, are depicted in~\tablename~\ref{tab:ablation}.

As illustrated in~\tablename~\ref{tab:ablation}, it is shown that GraphLoRA consistently outperforms seven variants in most cases, thereby demonstrating the effectiveness of each module of GraphLoRA.
Specifically, the most significant performance decline is observed for "w/o nfa" and "w/o smmd," emphasizing the importance of considering the discrepancy in feature distributions in transfer learning. This observation further validates the effectiveness of our proposed Structure-aware Maximum Mean Discrepancy for measuring the discrepancy in node feature distributions. Moreover, "w/o str" exhibits a more significant performance decline in the 10-shot setting compared to the public setting, indicating that the structure-aware regularization indeed contributes to improving the adaptability of pre-trained GNNs in scenarios with scarce labels.

\subsection{Catastrophic Forgetting Analysis (RQ 4)}
\begin{table}[]
\setlength{\tabcolsep}{1mm}
\caption{Catastrophic forgetting analysis. After pre-training on the PubMed dataset, we fine-tune the model on other datasets and then test it back on the PubMed dataset.
}
\label{tab:cata_forget_ana}
\resizebox{\linewidth}{!}{
\begin{tabular}{c|ccccc}
\hline
Method    & PM & PM$\rightarrow$CS$\rightarrow$PM   & PM$\rightarrow$C$\rightarrow$PM       & PM$\rightarrow$P$\rightarrow$PM      & PM$\rightarrow$Com$\rightarrow$PM  \\ \hline
FT           & 79.52±0.16 & 71.80±1.32 & 75.34±1.34 & 64.90±1.54 & 59.58±2.61 \\
GTOT         & 76.48±1.12 & 78.28±0.33 & 72.28±0.94 & 67.82±2.50 & 61.58±2.07 \\
AdapterGNN   & 76.44±0.97 & 77.50±0.68 & 76.50±0.73 & 74.06±1.95 & 73.34±1.56 \\
GraphControl & 78.30±0.43 & 78.00±0.27 & 75.16±2.09 & 65.48±1.02 & 65.42±3.29 \\
GraphLoRA    & \textbf{80.86±0.39} & \textbf{79.84±0.28} & \textbf{79.82±0.24} & \textbf{80.06±0.28} & \textbf{79.88±0.28} \\ \hline
\end{tabular}}
\end{table}

Fine-tuning the pretrained model with full parameters often leads to in catastrophic forgetting. To mitigate this, we freeze the pretrained parameters and introduce additional tunable parameters. To evaluate GraphLoRA's ability to alleviate catastrophic forgetting, we first pre-train the model on the PubMed dataset, then fine-tune it on other datasets, and finally assess its performance back on PubMed. Experimental results comparing GraphLoRA with full parameter fine-tuning (FT) and other baselines are presented in~\tablename~\ref{tab:cata_forget_ana}. FT exhibits a significant performance decline, whereas GraphLoRA shows only a marginal decrease. Other baselines also experience performance declines, while less severe than that of FT. This is attributed to the fact that these methods freeze the pre-trained parameters while introducing additional trainable parameters, thus mitigating the issue of catastrophic forgetting to some degree. Moreover, GraphLoRA significantly outperforms FT (average 18.64\%), highlighting its effectiveness in mitigating catastrophic forgetting.

\subsection{Visualization of Representations (RQ 5)}
In addition to quantitative analysis, we employ the t-SNE~\cite{van2008visualizing} method to visually assess the GraphLoRA's performance by visualizing the learned node embeddings on the CiteSeer dataset in the 10-shot setting. Specifically, \figurename~\ref{fig:GRACE_emb} shows embeddings learned by GRACE, while \figurename~\ref{fig:GRACEt_emb} and \figurename~\ref{fig:GraphLoRA_emb} display embeddings learned by GRACE$_t$ and GraphLoRA, respectively, following pre-training on the PubMed dataset. Each point in these figures represents a node, with its color denoting its label. From~\figurename~\ref{fig:visualization}, we observe that, compared to GRACE, the embeddings learned by GRACE$_t$ exhibit more blurred class boundaries, whereas the embeddings learned by GraphLoRA present clearer class boundaries. This observation suggests that GraphLoRA has a stronger capacity for learning node embeddings, proving beneficial for downstream tasks.

\begin{figure}[!t]
  \begin{subfigure}{0.15\textwidth}
    \includegraphics[width=\linewidth]{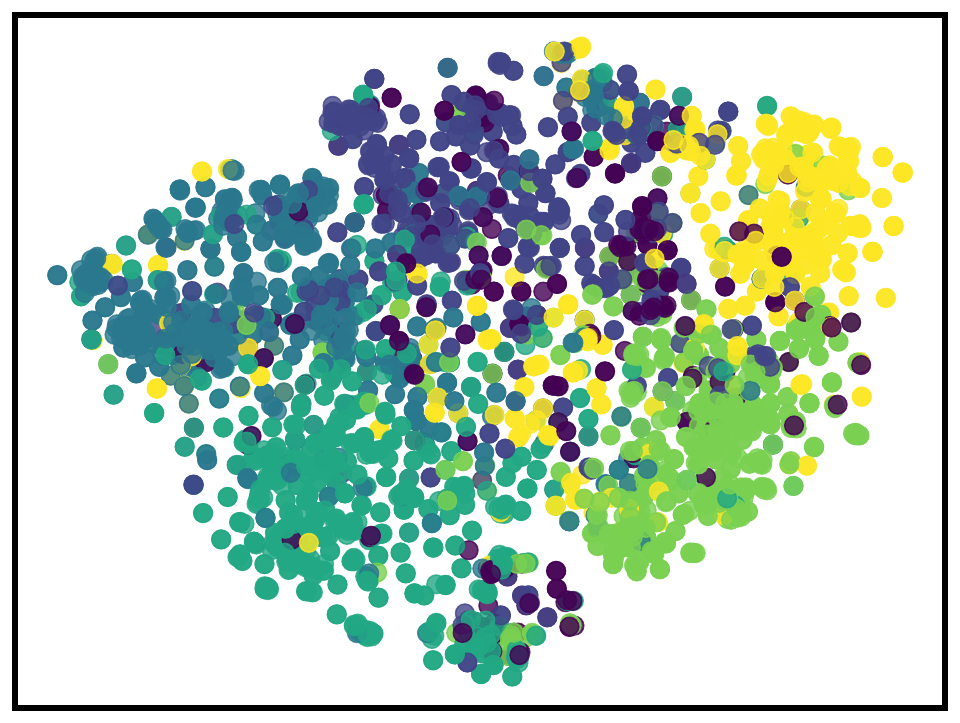}
    \caption{GRACE}
    \label{fig:GRACE_emb}
  \end{subfigure}
  \begin{subfigure}{0.15\textwidth}
    \includegraphics[width=\linewidth]{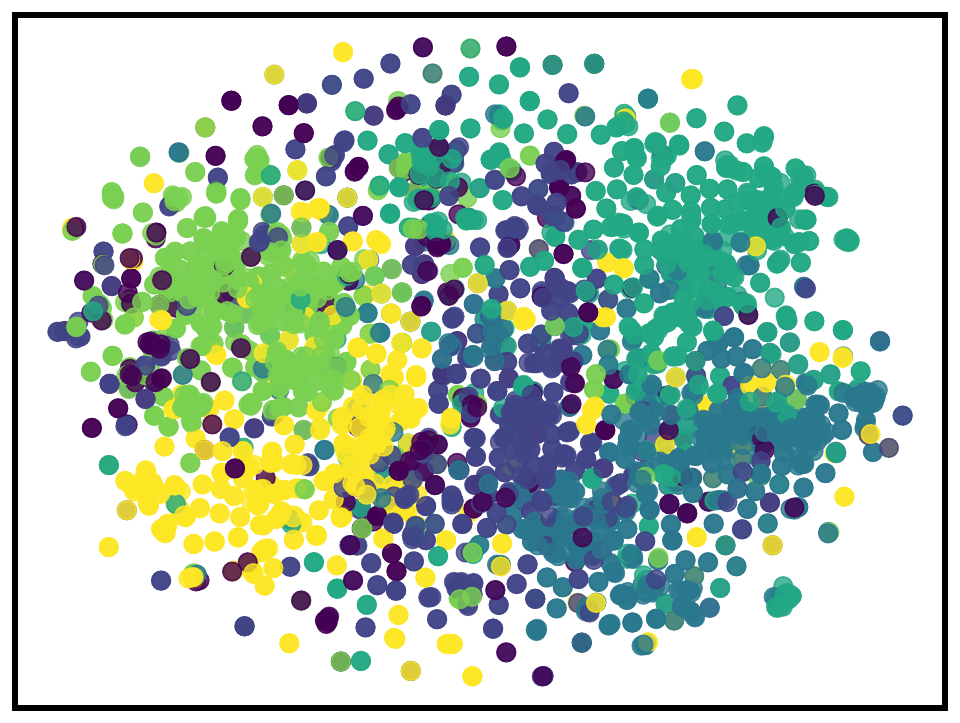}
    \caption{GRACE$_t$}
    \label{fig:GRACEt_emb}
  \end{subfigure}
  \begin{subfigure}{0.15\textwidth}
    \includegraphics[width=\linewidth]{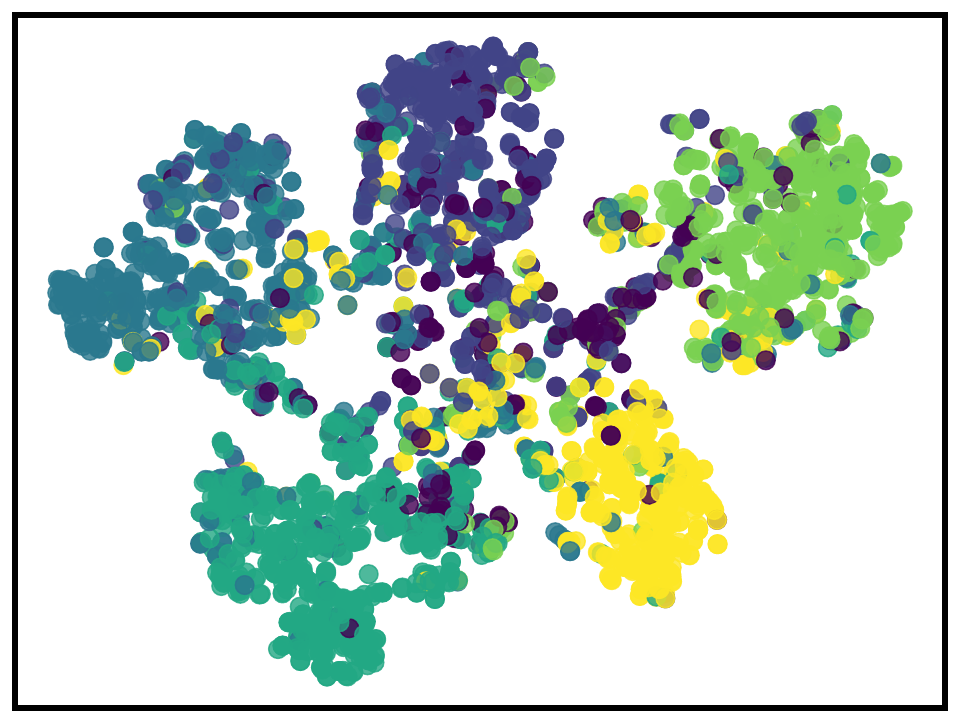}
    \caption{GraphLoRA}
    \label{fig:GraphLoRA_emb}
  \end{subfigure}
  \caption{Visualization of node embeddings on CiteSeer.}
  \label{fig:visualization}
\end{figure}

\section{Conclusion}
In this paper, we investigate the challenging problem of cross-graph transfer in graph neural networks. Inspired by the success of LoRA in fine-tuning large language models, we propose GraphLoRA, a parameter-efficient framework for fine-tuning pre-trained GNNs. Specifically, we introduce the node feature adaptation and structural knowledge transfer learning modules to address discrepancies in node feature distribution and graph structure between the source and target graphs. Additionally, a structure-aware regularization objective is proposed to improve adaptability in scenarios with limited labels. Theoretical analysis demonstrates that GraphLoRA has powerful representation capabilities and can fit any target GNN under mild conditions. Extensive experiments validate the effectiveness of GraphLoRA, even across disparate graph domains. Future work will focus on developing more efficient graph transfer learning methods to enhance computational efficiency and investigating its applicability to heterogeneous graphs, thereby broadening its generalizability to various graph types.

\begin{acks}
This work was supported by the National Natural Science Foundation of China (Grant No. 62276277 and 92370204), the Guangdong Basic and Applied Basic Research Foundation (Grant No. 2022B1515120059), the National Key R\&D Program of China (Grant No. 2023YFF0725004), the Guangzhou-HKUST(GZ) Joint Funding Program (Grant No. 2023A03J0008), and the Education Bureau of Guangzhou Municipality.
\end{acks}

\bibliographystyle{ACM-Reference-Format}
\bibliography{reference}


\begin{thebibliography}{75}

%%% the Web address.

\ifx \showCODEN    \undefined \def \showCODEN     #1{\unskip}     \fi
\ifx \showDOI      \undefined \def \showDOI       #1{#1}\fi
\ifx \showISBNx    \undefined \def \showISBNx     #1{\unskip}     \fi
\ifx \showISBNxiii \undefined \def \showISBNxiii  #1{\unskip}     \fi
\ifx \showISSN     \undefined \def \showISSN      #1{\unskip}     \fi
\ifx \showLCCN     \undefined \def \showLCCN      #1{\unskip}     \fi
\ifx \shownote     \undefined \def \shownote      #1{#1}          \fi
\ifx \showarticletitle \undefined \def \showarticletitle #1{#1}   \fi
\ifx \showURL      \undefined \def \showURL       {\relax}        \fi
% invisible to TeX
\providecommand\bibfield[2]{#2}
\providecommand\bibinfo[2]{#2}
\providecommand\natexlab[1]{#1}
\providecommand\showeprint[2][]{arXiv:#2}

\bibitem[Andersen et~al\mbox{.}(2006)]%
        {DBLP:conf/focs/AndersenCL06}
\bibfield{author}{\bibinfo{person}{Reid Andersen}, \bibinfo{person}{Fan R.~K. Chung}, {and} \bibinfo{person}{Kevin~J. Lang}.} \bibinfo{year}{2006}\natexlab{}.
\newblock \showarticletitle{Local Graph Partitioning using PageRank Vectors}. In \bibinfo{booktitle}{\emph{{FOCS}}}. \bibinfo{pages}{475--486}.
\newblock


\bibitem[Bao et~al\mbox{.}(2023)]%
        {DBLP:conf/recsys/BaoZZWF023}
\bibfield{author}{\bibinfo{person}{Keqin Bao}, \bibinfo{person}{Jizhi Zhang}, \bibinfo{person}{Yang Zhang}, \bibinfo{person}{Wenjie Wang}, \bibinfo{person}{Fuli Feng}, {and} \bibinfo{person}{Xiangnan He}.} \bibinfo{year}{2023}\natexlab{}.
\newblock \showarticletitle{TALLRec: An Effective and Efficient Tuning Framework to Align Large Language Model with Recommendation}. In \bibinfo{booktitle}{\emph{RecSys}}. \bibinfo{pages}{1007--1014}.
\newblock


\bibitem[Ben{-}David et~al\mbox{.}(2010)]%
        {DBLP:journals/ml/Ben-DavidBCKPV10}
\bibfield{author}{\bibinfo{person}{Shai Ben{-}David}, \bibinfo{person}{John Blitzer}, \bibinfo{person}{Koby Crammer}, \bibinfo{person}{Alex Kulesza}, \bibinfo{person}{Fernando Pereira}, {and} \bibinfo{person}{Jennifer~Wortman Vaughan}.} \bibinfo{year}{2010}\natexlab{}.
\newblock \showarticletitle{A theory of learning from different domains}.
\newblock \bibinfo{journal}{\emph{Mach. Learn.}} \bibinfo{volume}{79}, \bibinfo{number}{1-2} (\bibinfo{year}{2010}), \bibinfo{pages}{151--175}.
\newblock


\bibitem[Borgwardt et~al\mbox{.}(2006)]%
        {DBLP:conf/ismb/BorgwardtGRKSS06}
\bibfield{author}{\bibinfo{person}{Karsten~M. Borgwardt}, \bibinfo{person}{Arthur Gretton}, \bibinfo{person}{Malte~J. Rasch}, \bibinfo{person}{Hans{-}Peter Kriegel}, \bibinfo{person}{Bernhard Sch{\"{o}}lkopf}, {and} \bibinfo{person}{Alexander~J. Smola}.} \bibinfo{year}{2006}\natexlab{}.
\newblock \showarticletitle{Integrating structured biological data by Kernel Maximum Mean Discrepancy}. In \bibinfo{booktitle}{\emph{{ISMB} (Supplement of Bioinformatics)}}. \bibinfo{pages}{49--57}.
\newblock


\bibitem[Cao et~al\mbox{.}(2023)]%
        {DBLP:conf/kdd/CaoXYWZ0C023}
\bibfield{author}{\bibinfo{person}{Yuxuan Cao}, \bibinfo{person}{Jiarong Xu}, \bibinfo{person}{Carl~J. Yang}, \bibinfo{person}{Jiaan Wang}, \bibinfo{person}{Yunchao Zhang}, \bibinfo{person}{Chunping Wang}, \bibinfo{person}{Lei Chen}, {and} \bibinfo{person}{Yang Yang}.} \bibinfo{year}{2023}\natexlab{}.
\newblock \showarticletitle{When to Pre-Train Graph Neural Networks? From Data Generation Perspective!}. In \bibinfo{booktitle}{\emph{{KDD}}}. \bibinfo{pages}{142--153}.
\newblock


\bibitem[Chen et~al\mbox{.}(2020b)]%
        {DBLP:conf/icml/ChenWHDL20}
\bibfield{author}{\bibinfo{person}{Ming Chen}, \bibinfo{person}{Zhewei Wei}, \bibinfo{person}{Zengfeng Huang}, \bibinfo{person}{Bolin Ding}, {and} \bibinfo{person}{Yaliang Li}.} \bibinfo{year}{2020}\natexlab{b}.
\newblock \showarticletitle{Simple and Deep Graph Convolutional Networks}. In \bibinfo{booktitle}{\emph{{ICML}}}. \bibinfo{pages}{1725--1735}.
\newblock


\bibitem[Chen et~al\mbox{.}(2020a)]%
        {DBLP:conf/icml/ChenK0H20}
\bibfield{author}{\bibinfo{person}{Ting Chen}, \bibinfo{person}{Simon Kornblith}, \bibinfo{person}{Mohammad Norouzi}, {and} \bibinfo{person}{Geoffrey~E. Hinton}.} \bibinfo{year}{2020}\natexlab{a}.
\newblock \showarticletitle{A Simple Framework for Contrastive Learning of Visual Representations}. In \bibinfo{booktitle}{\emph{{ICML}}}. \bibinfo{pages}{1597--1607}.
\newblock


\bibitem[Dai et~al\mbox{.}(2023)]%
        {DBLP:journals/tkde/DaiWXSW23}
\bibfield{author}{\bibinfo{person}{Quanyu Dai}, \bibinfo{person}{Xiao{-}Ming Wu}, \bibinfo{person}{Jiaren Xiao}, \bibinfo{person}{Xiao Shen}, {and} \bibinfo{person}{Dan Wang}.} \bibinfo{year}{2023}\natexlab{}.
\newblock \showarticletitle{Graph Transfer Learning via Adversarial Domain Adaptation With Graph Convolution}.
\newblock \bibinfo{journal}{\emph{{IEEE} Trans. Knowl. Data Eng.}} \bibinfo{volume}{35}, \bibinfo{number}{5} (\bibinfo{year}{2023}), \bibinfo{pages}{4908--4922}.
\newblock


\bibitem[Fang et~al\mbox{.}(2023)]%
        {DBLP:conf/nips/FangZYWC23}
\bibfield{author}{\bibinfo{person}{Taoran Fang}, \bibinfo{person}{Yunchao Zhang}, \bibinfo{person}{Yang Yang}, \bibinfo{person}{Chunping Wang}, {and} \bibinfo{person}{Lei Chen}.} \bibinfo{year}{2023}\natexlab{}.
\newblock \showarticletitle{Universal Prompt Tuning for Graph Neural Networks}. In \bibinfo{booktitle}{\emph{NeurIPS}}.
\newblock


\bibitem[Gilmer et~al\mbox{.}(2017)]%
        {DBLP:conf/icml/GilmerSRVD17}
\bibfield{author}{\bibinfo{person}{Justin Gilmer}, \bibinfo{person}{Samuel~S. Schoenholz}, \bibinfo{person}{Patrick~F. Riley}, \bibinfo{person}{Oriol Vinyals}, {and} \bibinfo{person}{George~E. Dahl}.} \bibinfo{year}{2017}\natexlab{}.
\newblock \showarticletitle{Neural Message Passing for Quantum Chemistry}. In \bibinfo{booktitle}{\emph{{ICML}}}. \bibinfo{pages}{1263--1272}.
\newblock


\bibitem[Gretton et~al\mbox{.}(2012)]%
        {DBLP:journals/jmlr/GrettonBRSS12}
\bibfield{author}{\bibinfo{person}{Arthur Gretton}, \bibinfo{person}{Karsten~M. Borgwardt}, \bibinfo{person}{Malte~J. Rasch}, \bibinfo{person}{Bernhard Sch{\"{o}}lkopf}, {and} \bibinfo{person}{Alexander~J. Smola}.} \bibinfo{year}{2012}\natexlab{}.
\newblock \showarticletitle{A Kernel Two-Sample Test}.
\newblock \bibinfo{journal}{\emph{J. Mach. Learn. Res.}} \bibinfo{volume}{13}, \bibinfo{number}{1} (\bibinfo{year}{2012}), \bibinfo{pages}{723--773}.
\newblock


\bibitem[Gui et~al\mbox{.}(2024)]%
        {DBLP:conf/aaai/GuiYX24}
\bibfield{author}{\bibinfo{person}{Anchun Gui}, \bibinfo{person}{Jinqiang Ye}, {and} \bibinfo{person}{Han Xiao}.} \bibinfo{year}{2024}\natexlab{}.
\newblock \showarticletitle{G-Adapter: Towards Structure-Aware Parameter-Efficient Transfer Learning for Graph Transformer Networks}. In \bibinfo{booktitle}{\emph{{AAAI}}}. \bibinfo{pages}{12226--12234}.
\newblock


\bibitem[Guo et~al\mbox{.}(2024)]%
        {DBLP:journals/corr/abs-2402-08228}
\bibfield{author}{\bibinfo{person}{Kai Guo}, \bibinfo{person}{Hongzhi Wen}, \bibinfo{person}{Wei Jin}, \bibinfo{person}{Yaming Guo}, \bibinfo{person}{Jiliang Tang}, {and} \bibinfo{person}{Yi Chang}.} \bibinfo{year}{2024}\natexlab{}.
\newblock \showarticletitle{Investigating Out-of-Distribution Generalization of GNNs: An Architecture Perspective}.
\newblock \bibinfo{journal}{\emph{CoRR}}  \bibinfo{volume}{abs/2402.08228} (\bibinfo{year}{2024}).
\newblock


\bibitem[Hamilton et~al\mbox{.}(2017)]%
        {DBLP:conf/nips/HamiltonYL17}
\bibfield{author}{\bibinfo{person}{William~L. Hamilton}, \bibinfo{person}{Zhitao Ying}, {and} \bibinfo{person}{Jure Leskovec}.} \bibinfo{year}{2017}\natexlab{}.
\newblock \showarticletitle{Inductive Representation Learning on Large Graphs}. In \bibinfo{booktitle}{\emph{{NeurIPS}}}. \bibinfo{pages}{1024--1034}.
\newblock


\bibitem[Han et~al\mbox{.}(2021)]%
        {DBLP:conf/kdd/HanHAB21}
\bibfield{author}{\bibinfo{person}{Xueting Han}, \bibinfo{person}{Zhenhuan Huang}, \bibinfo{person}{Bang An}, {and} \bibinfo{person}{Jing Bai}.} \bibinfo{year}{2021}\natexlab{}.
\newblock \showarticletitle{Adaptive Transfer Learning on Graph Neural Networks}. In \bibinfo{booktitle}{\emph{{KDD}}}. \bibinfo{pages}{565--574}.
\newblock


\bibitem[Hassani and Ahmadi(2020)]%
        {DBLP:conf/icml/HassaniA20}
\bibfield{author}{\bibinfo{person}{Kaveh Hassani} {and} \bibinfo{person}{Amir Hosein~Khas Ahmadi}.} \bibinfo{year}{2020}\natexlab{}.
\newblock \showarticletitle{Contrastive Multi-View Representation Learning on Graphs}. In \bibinfo{booktitle}{\emph{{ICML}}}. \bibinfo{pages}{4116--4126}.
\newblock


\bibitem[Houlsby et~al\mbox{.}(2019)]%
        {DBLP:conf/icml/HoulsbyGJMLGAG19}
\bibfield{author}{\bibinfo{person}{Neil Houlsby}, \bibinfo{person}{Andrei Giurgiu}, \bibinfo{person}{Stanislaw Jastrzebski}, \bibinfo{person}{Bruna Morrone}, \bibinfo{person}{Quentin de Laroussilhe}, \bibinfo{person}{Andrea Gesmundo}, \bibinfo{person}{Mona Attariyan}, {and} \bibinfo{person}{Sylvain Gelly}.} \bibinfo{year}{2019}\natexlab{}.
\newblock \showarticletitle{Parameter-Efficient Transfer Learning for {NLP}}. In \bibinfo{booktitle}{\emph{{ICML}}}. \bibinfo{pages}{2790--2799}.
\newblock


\bibitem[Hu et~al\mbox{.}(2022)]%
        {DBLP:conf/iclr/HuSWALWWC22}
\bibfield{author}{\bibinfo{person}{Edward~J. Hu}, \bibinfo{person}{Yelong Shen}, \bibinfo{person}{Phillip Wallis}, \bibinfo{person}{Zeyuan Allen{-}Zhu}, \bibinfo{person}{Yuanzhi Li}, \bibinfo{person}{Shean Wang}, \bibinfo{person}{Lu Wang}, {and} \bibinfo{person}{Weizhu Chen}.} \bibinfo{year}{2022}\natexlab{}.
\newblock \showarticletitle{LoRA: Low-Rank Adaptation of Large Language Models}. In \bibinfo{booktitle}{\emph{{ICLR}}}.
\newblock


\bibitem[Hu et~al\mbox{.}(2020)]%
        {DBLP:conf/nips/HuFZDRLCL20}
\bibfield{author}{\bibinfo{person}{Weihua Hu}, \bibinfo{person}{Matthias Fey}, \bibinfo{person}{Marinka Zitnik}, \bibinfo{person}{Yuxiao Dong}, \bibinfo{person}{Hongyu Ren}, \bibinfo{person}{Bowen Liu}, \bibinfo{person}{Michele Catasta}, {and} \bibinfo{person}{Jure Leskovec}.} \bibinfo{year}{2020}\natexlab{}.
\newblock \showarticletitle{Open Graph Benchmark: Datasets for Machine Learning on Graphs}. In \bibinfo{booktitle}{\emph{NeurIPS}}.
\newblock


\bibitem[Hwang et~al\mbox{.}(2020)]%
        {DBLP:conf/nips/HwangPKKHK20}
\bibfield{author}{\bibinfo{person}{Dasol Hwang}, \bibinfo{person}{Jinyoung Park}, \bibinfo{person}{Sunyoung Kwon}, \bibinfo{person}{Kyung{-}Min Kim}, \bibinfo{person}{Jung{-}Woo Ha}, {and} \bibinfo{person}{Hyunwoo~J. Kim}.} \bibinfo{year}{2020}\natexlab{}.
\newblock \showarticletitle{Self-supervised Auxiliary Learning with Meta-paths for Heterogeneous Graphs}. In \bibinfo{booktitle}{\emph{NeurIPS}}. \bibinfo{pages}{10294--10305}.
\newblock


\bibitem[Jiang(2021)]%
        {DBLP:conf/kdd/Jiang21}
\bibfield{author}{\bibinfo{person}{Meng Jiang}.} \bibinfo{year}{2021}\natexlab{}.
\newblock \showarticletitle{Cross-Network Learning with Partially Aligned Graph Convolutional Networks}. In \bibinfo{booktitle}{\emph{{KDD}}}. \bibinfo{pages}{746--755}.
\newblock


\bibitem[Kingma and Ba(2015)]%
        {DBLP:journals/corr/KingmaB14}
\bibfield{author}{\bibinfo{person}{Diederik~P. Kingma} {and} \bibinfo{person}{Jimmy Ba}.} \bibinfo{year}{2015}\natexlab{}.
\newblock \showarticletitle{Adam: {A} Method for Stochastic Optimization}. In \bibinfo{booktitle}{\emph{{ICLR} (Poster)}}.
\newblock


\bibitem[Kipf and Welling(2017)]%
        {DBLP:conf/iclr/KipfW17}
\bibfield{author}{\bibinfo{person}{Thomas~N. Kipf} {and} \bibinfo{person}{Max Welling}.} \bibinfo{year}{2017}\natexlab{}.
\newblock \showarticletitle{Semi-Supervised Classification with Graph Convolutional Networks}. In \bibinfo{booktitle}{\emph{{ICLR} (Poster)}}.
\newblock


\bibitem[Klicpera et~al\mbox{.}(2019)]%
        {DBLP:conf/nips/KlicperaWG19}
\bibfield{author}{\bibinfo{person}{Johannes Klicpera}, \bibinfo{person}{Stefan Wei{\ss}enberger}, {and} \bibinfo{person}{Stephan G{\"{u}}nnemann}.} \bibinfo{year}{2019}\natexlab{}.
\newblock \showarticletitle{Diffusion Improves Graph Learning}. In \bibinfo{booktitle}{\emph{NeurIPS}}. \bibinfo{pages}{13333--13345}.
\newblock


\bibitem[Lachi et~al\mbox{.}(2024)]%
        {lachi2024graphfm}
\bibfield{author}{\bibinfo{person}{Divyansha Lachi}, \bibinfo{person}{Mehdi Azabou}, \bibinfo{person}{Vinam Arora}, {and} \bibinfo{person}{Eva Dyer}.} \bibinfo{year}{2024}\natexlab{}.
\newblock \showarticletitle{GraphFM: A Scalable Framework for Multi-Graph Pretraining}.
\newblock \bibinfo{journal}{\emph{arXiv preprint arXiv:2407.11907}} (\bibinfo{year}{2024}).
\newblock


\bibitem[Lai et~al\mbox{.}(2024)]%
        {DBLP:journals/kais/LaiYDLW24}
\bibfield{author}{\bibinfo{person}{Pei{-}Yuan Lai}, \bibinfo{person}{Zhe{-}Rui Yang}, \bibinfo{person}{Qing{-}Yun Dai}, \bibinfo{person}{De{-}Zhang Liao}, {and} \bibinfo{person}{Chang{-}Dong Wang}.} \bibinfo{year}{2024}\natexlab{}.
\newblock \showarticletitle{BiMuF: a bi-directional recommender system with multi-semantic filter for online recruitment}.
\newblock \bibinfo{journal}{\emph{Knowl. Inf. Syst.}} \bibinfo{volume}{66}, \bibinfo{number}{3} (\bibinfo{year}{2024}), \bibinfo{pages}{1751--1776}.
\newblock


\bibitem[Lester et~al\mbox{.}(2021)]%
        {DBLP:conf/emnlp/LesterAC21}
\bibfield{author}{\bibinfo{person}{Brian Lester}, \bibinfo{person}{Rami Al{-}Rfou}, {and} \bibinfo{person}{Noah Constant}.} \bibinfo{year}{2021}\natexlab{}.
\newblock \showarticletitle{The Power of Scale for Parameter-Efficient Prompt Tuning}. In \bibinfo{booktitle}{\emph{{EMNLP} {(1)}}}. \bibinfo{pages}{3045--3059}.
\newblock


\bibitem[Li et~al\mbox{.}(2024)]%
        {DBLP:conf/aaai/LiH024}
\bibfield{author}{\bibinfo{person}{Shengrui Li}, \bibinfo{person}{Xueting Han}, {and} \bibinfo{person}{Jing Bai}.} \bibinfo{year}{2024}\natexlab{}.
\newblock \showarticletitle{AdapterGNN: Parameter-Efficient Fine-Tuning Improves Generalization in GNNs}. In \bibinfo{booktitle}{\emph{{AAAI}}}. \bibinfo{pages}{13600--13608}.
\newblock


\bibitem[Li et~al\mbox{.}(2023a)]%
        {DBLP:conf/kdd/LiWXL23}
\bibfield{author}{\bibinfo{person}{Wen{-}Zhi Li}, \bibinfo{person}{Chang{-}Dong Wang}, \bibinfo{person}{Hui Xiong}, {and} \bibinfo{person}{Jian{-}Huang Lai}.} \bibinfo{year}{2023}\natexlab{a}.
\newblock \showarticletitle{GraphSHA: Synthesizing Harder Samples for Class-Imbalanced Node Classification}. In \bibinfo{booktitle}{\emph{{KDD}}}. \bibinfo{pages}{1328--1340}.
\newblock


\bibitem[Li et~al\mbox{.}(2023b)]%
        {DBLP:conf/kdd/LiWXL23a}
\bibfield{author}{\bibinfo{person}{Wen{-}Zhi Li}, \bibinfo{person}{Chang{-}Dong Wang}, \bibinfo{person}{Hui Xiong}, {and} \bibinfo{person}{Jian{-}Huang Lai}.} \bibinfo{year}{2023}\natexlab{b}.
\newblock \showarticletitle{HomoGCL: Rethinking Homophily in Graph Contrastive Learning}. In \bibinfo{booktitle}{\emph{{KDD}}}. \bibinfo{pages}{1341--1352}.
\newblock


\bibitem[Li and Liang(2021)]%
        {DBLP:conf/acl/LiL20}
\bibfield{author}{\bibinfo{person}{Xiang~Lisa Li} {and} \bibinfo{person}{Percy Liang}.} \bibinfo{year}{2021}\natexlab{}.
\newblock \showarticletitle{Prefix-Tuning: Optimizing Continuous Prompts for Generation}. In \bibinfo{booktitle}{\emph{{ACL/IJCNLP} {(1)}}}. \bibinfo{pages}{4582--4597}.
\newblock


\bibitem[Lin et~al\mbox{.}(2020)]%
        {DBLP:conf/emnlp/LinMF20}
\bibfield{author}{\bibinfo{person}{Zhaojiang Lin}, \bibinfo{person}{Andrea Madotto}, {and} \bibinfo{person}{Pascale Fung}.} \bibinfo{year}{2020}\natexlab{}.
\newblock \showarticletitle{Exploring Versatile Generative Language Model Via Parameter-Efficient Transfer Learning}. In \bibinfo{booktitle}{\emph{{EMNLP} (Findings)}}. \bibinfo{pages}{441--459}.
\newblock


\bibitem[Liu et~al\mbox{.}(2022)]%
        {DBLP:conf/nips/LiuLJ22}
\bibfield{author}{\bibinfo{person}{Fan Liu}, \bibinfo{person}{Hao Liu}, {and} \bibinfo{person}{Wenzhao Jiang}.} \bibinfo{year}{2022}\natexlab{}.
\newblock \showarticletitle{Practical Adversarial Attacks on Spatiotemporal Traffic Forecasting Models}. In \bibinfo{booktitle}{\emph{NeurIPS}}.
\newblock


\bibitem[Liu et~al\mbox{.}(2021a)]%
        {DBLP:conf/kdd/LiuGLLXCWYZDDX21}
\bibfield{author}{\bibinfo{person}{Hao Liu}, \bibinfo{person}{Qian Gao}, \bibinfo{person}{Jiang Li}, \bibinfo{person}{Xiaochao Liao}, \bibinfo{person}{Hao Xiong}, \bibinfo{person}{Guangxing Chen}, \bibinfo{person}{Wenlin Wang}, \bibinfo{person}{Guobao Yang}, \bibinfo{person}{Zhiwei Zha}, \bibinfo{person}{Daxiang Dong}, \bibinfo{person}{Dejing Dou}, {and} \bibinfo{person}{Haoyi Xiong}.} \bibinfo{year}{2021}\natexlab{a}.
\newblock \showarticletitle{{JIZHI:} {A} Fast and Cost-Effective Model-As-A-Service System for Web-Scale Online Inference at Baidu}. In \bibinfo{booktitle}{\emph{{KDD}}}. \bibinfo{pages}{3289--3298}.
\newblock


\bibitem[Liu et~al\mbox{.}(2023a)]%
        {DBLP:conf/icml/LiuLFTZQL23}
\bibfield{author}{\bibinfo{person}{Shikun Liu}, \bibinfo{person}{Tianchun Li}, \bibinfo{person}{Yongbin Feng}, \bibinfo{person}{Nhan Tran}, \bibinfo{person}{Han Zhao}, \bibinfo{person}{Qiang Qiu}, {and} \bibinfo{person}{Pan Li}.} \bibinfo{year}{2023}\natexlab{a}.
\newblock \showarticletitle{Structural Re-weighting Improves Graph Domain Adaptation}. In \bibinfo{booktitle}{\emph{{ICML}}}. \bibinfo{pages}{21778--21793}.
\newblock


\bibitem[Liu et~al\mbox{.}(2021b)]%
        {DBLP:journals/corr/abs-2103-10385}
\bibfield{author}{\bibinfo{person}{Xiao Liu}, \bibinfo{person}{Yanan Zheng}, \bibinfo{person}{Zhengxiao Du}, \bibinfo{person}{Ming Ding}, \bibinfo{person}{Yujie Qian}, \bibinfo{person}{Zhilin Yang}, {and} \bibinfo{person}{Jie Tang}.} \bibinfo{year}{2021}\natexlab{b}.
\newblock \showarticletitle{{GPT} Understands, Too}.
\newblock \bibinfo{journal}{\emph{CoRR}}  \bibinfo{volume}{abs/2103.10385} (\bibinfo{year}{2021}).
\newblock


\bibitem[Liu et~al\mbox{.}(2023b)]%
        {DBLP:conf/www/LiuY0023}
\bibfield{author}{\bibinfo{person}{Zemin Liu}, \bibinfo{person}{Xingtong Yu}, \bibinfo{person}{Yuan Fang}, {and} \bibinfo{person}{Xinming Zhang}.} \bibinfo{year}{2023}\natexlab{b}.
\newblock \showarticletitle{GraphPrompt: Unifying Pre-Training and Downstream Tasks for Graph Neural Networks}. In \bibinfo{booktitle}{\emph{{WWW}}}. \bibinfo{pages}{417--428}.
\newblock


\bibitem[Long et~al\mbox{.}(2023)]%
        {DBLP:conf/emnlp/LongWP23}
\bibfield{author}{\bibinfo{person}{Quanyu Long}, \bibinfo{person}{Wenya Wang}, {and} \bibinfo{person}{Sinno~Jialin Pan}.} \bibinfo{year}{2023}\natexlab{}.
\newblock \showarticletitle{Adapt in Contexts: Retrieval-Augmented Domain Adaptation via In-Context Learning}. In \bibinfo{booktitle}{\emph{{EMNLP}}}. \bibinfo{pages}{6525--6542}.
\newblock


\bibitem[Mao et~al\mbox{.}(2021)]%
        {DBLP:journals/corr/abs-2112-00955}
\bibfield{author}{\bibinfo{person}{Haitao Mao}, \bibinfo{person}{Lun Du}, \bibinfo{person}{Yujia Zheng}, \bibinfo{person}{Qiang Fu}, \bibinfo{person}{Zelin Li}, \bibinfo{person}{Xu Chen}, \bibinfo{person}{Shi Han}, {and} \bibinfo{person}{Dongmei Zhang}.} \bibinfo{year}{2021}\natexlab{}.
\newblock \showarticletitle{Source Free Unsupervised Graph Domain Adaptation}.
\newblock \bibinfo{journal}{\emph{CoRR}}  \bibinfo{volume}{abs/2112.00955} (\bibinfo{year}{2021}).
\newblock


\bibitem[McAuley et~al\mbox{.}(2015)]%
        {DBLP:conf/sigir/McAuleyTSH15}
\bibfield{author}{\bibinfo{person}{Julian~J. McAuley}, \bibinfo{person}{Christopher Targett}, \bibinfo{person}{Qinfeng Shi}, {and} \bibinfo{person}{Anton van~den Hengel}.} \bibinfo{year}{2015}\natexlab{}.
\newblock \showarticletitle{Image-Based Recommendations on Styles and Substitutes}. In \bibinfo{booktitle}{\emph{{SIGIR}}}. \bibinfo{pages}{43--52}.
\newblock


\bibitem[McPherson et~al\mbox{.}(2001)]%
        {doi:10.1146/annurev.soc.27.1.415}
\bibfield{author}{\bibinfo{person}{Miller McPherson}, \bibinfo{person}{Lynn Smith-Lovin}, {and} \bibinfo{person}{James~M Cook}.} \bibinfo{year}{2001}\natexlab{}.
\newblock \showarticletitle{Birds of a Feather: Homophily in Social Networks}.
\newblock \bibinfo{journal}{\emph{Annual Review of Sociology}} \bibinfo{volume}{27}, \bibinfo{number}{1} (\bibinfo{year}{2001}), \bibinfo{pages}{415--444}.
\newblock


\bibitem[Ni et~al\mbox{.}(2018)]%
        {DBLP:conf/www/NiCLCCX018}
\bibfield{author}{\bibinfo{person}{Jingchao Ni}, \bibinfo{person}{Shiyu Chang}, \bibinfo{person}{Xiao Liu}, \bibinfo{person}{Wei Cheng}, \bibinfo{person}{Haifeng Chen}, \bibinfo{person}{Dongkuan Xu}, {and} \bibinfo{person}{Xiang Zhang}.} \bibinfo{year}{2018}\natexlab{}.
\newblock \showarticletitle{Co-Regularized Deep Multi-Network Embedding}. In \bibinfo{booktitle}{\emph{{WWW}}}. \bibinfo{pages}{469--478}.
\newblock


\bibitem[Page et~al\mbox{.}(1998)]%
        {page1998pagerank}
\bibfield{author}{\bibinfo{person}{Lawrence Page}, \bibinfo{person}{Sergey Brin}, \bibinfo{person}{Rajeev Motwani}, {and} \bibinfo{person}{Terry Winograd}.} \bibinfo{year}{1998}\natexlab{}.
\newblock \bibinfo{booktitle}{\emph{The pagerank citation ranking: Bring order to the web}}.
\newblock \bibinfo{type}{{T}echnical {R}eport}.
\newblock


\bibitem[Pei et~al\mbox{.}(2020)]%
        {DBLP:conf/iclr/PeiWCLY20}
\bibfield{author}{\bibinfo{person}{Hongbin Pei}, \bibinfo{person}{Bingzhe Wei}, \bibinfo{person}{Kevin~Chen{-}Chuan Chang}, \bibinfo{person}{Yu Lei}, {and} \bibinfo{person}{Bo Yang}.} \bibinfo{year}{2020}\natexlab{}.
\newblock \showarticletitle{Geom-GCN: Geometric Graph Convolutional Networks}. In \bibinfo{booktitle}{\emph{{ICLR}}}.
\newblock


\bibitem[Pennington et~al\mbox{.}(2014)]%
        {DBLP:conf/emnlp/PenningtonSM14}
\bibfield{author}{\bibinfo{person}{Jeffrey Pennington}, \bibinfo{person}{Richard Socher}, {and} \bibinfo{person}{Christopher~D. Manning}.} \bibinfo{year}{2014}\natexlab{}.
\newblock \showarticletitle{Glove: Global Vectors for Word Representation}. In \bibinfo{booktitle}{\emph{{EMNLP}}}. \bibinfo{pages}{1532--1543}.
\newblock


\bibitem[Pfeiffer et~al\mbox{.}(2021)]%
        {DBLP:conf/eacl/PfeifferKRCG21}
\bibfield{author}{\bibinfo{person}{Jonas Pfeiffer}, \bibinfo{person}{Aishwarya Kamath}, \bibinfo{person}{Andreas R{\"{u}}ckl{\'{e}}}, \bibinfo{person}{Kyunghyun Cho}, {and} \bibinfo{person}{Iryna Gurevych}.} \bibinfo{year}{2021}\natexlab{}.
\newblock \showarticletitle{AdapterFusion: Non-Destructive Task Composition for Transfer Learning}. In \bibinfo{booktitle}{\emph{{EACL}}}. \bibinfo{pages}{487--503}.
\newblock


\bibitem[Qiu et~al\mbox{.}(2020)]%
        {DBLP:conf/kdd/QiuCDZYDWT20}
\bibfield{author}{\bibinfo{person}{Jiezhong Qiu}, \bibinfo{person}{Qibin Chen}, \bibinfo{person}{Yuxiao Dong}, \bibinfo{person}{Jing Zhang}, \bibinfo{person}{Hongxia Yang}, \bibinfo{person}{Ming Ding}, \bibinfo{person}{Kuansan Wang}, {and} \bibinfo{person}{Jie Tang}.} \bibinfo{year}{2020}\natexlab{}.
\newblock \showarticletitle{{GCC:} Graph Contrastive Coding for Graph Neural Network Pre-Training}. In \bibinfo{booktitle}{\emph{{KDD}}}. \bibinfo{pages}{1150--1160}.
\newblock


\bibitem[Shchur et~al\mbox{.}(2018)]%
        {DBLP:journals/corr/abs-1811-05868}
\bibfield{author}{\bibinfo{person}{Oleksandr Shchur}, \bibinfo{person}{Maximilian Mumme}, \bibinfo{person}{Aleksandar Bojchevski}, {and} \bibinfo{person}{Stephan G{\"{u}}nnemann}.} \bibinfo{year}{2018}\natexlab{}.
\newblock \showarticletitle{Pitfalls of Graph Neural Network Evaluation}.
\newblock \bibinfo{journal}{\emph{CoRR}}  \bibinfo{volume}{abs/1811.05868} (\bibinfo{year}{2018}).
\newblock


\bibitem[Srivastava(2010)]%
        {DBLP:journals/corr/abs-1012-4050}
\bibfield{author}{\bibinfo{person}{Abhishek Srivastava}.} \bibinfo{year}{2010}\natexlab{}.
\newblock \showarticletitle{Motif Analysis in the Amazon Product Co-Purchasing Network}.
\newblock \bibinfo{journal}{\emph{CoRR}}  \bibinfo{volume}{abs/1012.4050} (\bibinfo{year}{2010}).
\newblock


\bibitem[Sun et~al\mbox{.}(2022)]%
        {DBLP:conf/kdd/SunZHWW22}
\bibfield{author}{\bibinfo{person}{Mingchen Sun}, \bibinfo{person}{Kaixiong Zhou}, \bibinfo{person}{Xin He}, \bibinfo{person}{Ying Wang}, {and} \bibinfo{person}{Xin Wang}.} \bibinfo{year}{2022}\natexlab{}.
\newblock \showarticletitle{{GPPT:} Graph Pre-training and Prompt Tuning to Generalize Graph Neural Networks}. In \bibinfo{booktitle}{\emph{{KDD}}}. \bibinfo{pages}{1717--1727}.
\newblock


\bibitem[Sun et~al\mbox{.}(2023)]%
        {DBLP:conf/kdd/SunCLLG23}
\bibfield{author}{\bibinfo{person}{Xiangguo Sun}, \bibinfo{person}{Hong Cheng}, \bibinfo{person}{Jia Li}, \bibinfo{person}{Bo Liu}, {and} \bibinfo{person}{Jihong Guan}.} \bibinfo{year}{2023}\natexlab{}.
\newblock \showarticletitle{All in One: Multi-Task Prompting for Graph Neural Networks}. In \bibinfo{booktitle}{\emph{{KDD}}}. \bibinfo{pages}{2120--2131}.
\newblock


\bibitem[van~den Oord et~al\mbox{.}(2018)]%
        {DBLP:journals/corr/abs-1807-03748}
\bibfield{author}{\bibinfo{person}{A{\"{a}}ron van~den Oord}, \bibinfo{person}{Yazhe Li}, {and} \bibinfo{person}{Oriol Vinyals}.} \bibinfo{year}{2018}\natexlab{}.
\newblock \showarticletitle{Representation Learning with Contrastive Predictive Coding}.
\newblock \bibinfo{journal}{\emph{CoRR}}  \bibinfo{volume}{abs/1807.03748} (\bibinfo{year}{2018}).
\newblock


\bibitem[Van~der Maaten and Hinton(2008)]%
        {van2008visualizing}
\bibfield{author}{\bibinfo{person}{Laurens Van~der Maaten} {and} \bibinfo{person}{Geoffrey Hinton}.} \bibinfo{year}{2008}\natexlab{}.
\newblock \showarticletitle{Visualizing data using t-SNE.}
\newblock \bibinfo{journal}{\emph{Journal of machine learning research}} \bibinfo{volume}{9}, \bibinfo{number}{11} (\bibinfo{year}{2008}).
\newblock


\bibitem[Velickovic et~al\mbox{.}(2018)]%
        {DBLP:conf/iclr/VelickovicCCRLB18}
\bibfield{author}{\bibinfo{person}{Petar Velickovic}, \bibinfo{person}{Guillem Cucurull}, \bibinfo{person}{Arantxa Casanova}, \bibinfo{person}{Adriana Romero}, \bibinfo{person}{Pietro Li{\`{o}}}, {and} \bibinfo{person}{Yoshua Bengio}.} \bibinfo{year}{2018}\natexlab{}.
\newblock \showarticletitle{Graph Attention Networks}. In \bibinfo{booktitle}{\emph{{ICLR} (Poster)}}.
\newblock


\bibitem[Wang et~al\mbox{.}(2024)]%
        {DBLP:journals/corr/abs-2401-10394}
\bibfield{author}{\bibinfo{person}{Fali Wang}, \bibinfo{person}{Tianxiang Zhao}, {and} \bibinfo{person}{Suhang Wang}.} \bibinfo{year}{2024}\natexlab{}.
\newblock \showarticletitle{Distribution Consistency based Self-Training for Graph Neural Networks with Sparse Labels}.
\newblock \bibinfo{journal}{\emph{CoRR}}  \bibinfo{volume}{abs/2401.10394} (\bibinfo{year}{2024}).
\newblock


\bibitem[Wang et~al\mbox{.}(2018)]%
        {DBLP:conf/kdd/WangHZZZL18}
\bibfield{author}{\bibinfo{person}{Jizhe Wang}, \bibinfo{person}{Pipei Huang}, \bibinfo{person}{Huan Zhao}, \bibinfo{person}{Zhibo Zhang}, \bibinfo{person}{Binqiang Zhao}, {and} \bibinfo{person}{Dik~Lun Lee}.} \bibinfo{year}{2018}\natexlab{}.
\newblock \showarticletitle{Billion-scale Commodity Embedding for E-commerce Recommendation in Alibaba}. In \bibinfo{booktitle}{\emph{{KDD}}}. \bibinfo{pages}{839--848}.
\newblock


\bibitem[Wei et~al\mbox{.}(2018)]%
        {DBLP:conf/sigmod/WeiHX0SW18}
\bibfield{author}{\bibinfo{person}{Zhewei Wei}, \bibinfo{person}{Xiaodong He}, \bibinfo{person}{Xiaokui Xiao}, \bibinfo{person}{Sibo Wang}, \bibinfo{person}{Shuo Shang}, {and} \bibinfo{person}{Ji{-}Rong Wen}.} \bibinfo{year}{2018}\natexlab{}.
\newblock \showarticletitle{TopPPR: Top-k Personalized PageRank Queries with Precision Guarantees on Large Graphs}. In \bibinfo{booktitle}{\emph{{SIGMOD} Conference}}. \bibinfo{pages}{441--456}.
\newblock


\bibitem[Wu et~al\mbox{.}(2008)]%
        {DBLP:journals/jsea/WuHL08}
\bibfield{author}{\bibinfo{person}{Wenchen Wu}, \bibinfo{person}{Yanni Han}, {and} \bibinfo{person}{Deyi Li}.} \bibinfo{year}{2008}\natexlab{}.
\newblock \showarticletitle{Motif-based Classification in Journal Citation Networks}.
\newblock \bibinfo{journal}{\emph{J. Softw. Eng. Appl.}} \bibinfo{volume}{1}, \bibinfo{number}{1} (\bibinfo{year}{2008}), \bibinfo{pages}{53--59}.
\newblock


\bibitem[Xu et~al\mbox{.}(2023)]%
        {DBLP:journals/corr/abs-2312-12148}
\bibfield{author}{\bibinfo{person}{Lingling Xu}, \bibinfo{person}{Haoran Xie}, \bibinfo{person}{Si{-}Zhao~Joe Qin}, \bibinfo{person}{Xiaohui Tao}, {and} \bibinfo{person}{Fu~Lee Wang}.} \bibinfo{year}{2023}\natexlab{}.
\newblock \showarticletitle{Parameter-Efficient Fine-Tuning Methods for Pretrained Language Models: {A} Critical Review and Assessment}.
\newblock \bibinfo{journal}{\emph{CoRR}}  \bibinfo{volume}{abs/2312.12148} (\bibinfo{year}{2023}).
\newblock


\bibitem[Yang et~al\mbox{.}(2016)]%
        {DBLP:conf/icml/YangCS16}
\bibfield{author}{\bibinfo{person}{Zhilin Yang}, \bibinfo{person}{William~W. Cohen}, {and} \bibinfo{person}{Ruslan Salakhutdinov}.} \bibinfo{year}{2016}\natexlab{}.
\newblock \showarticletitle{Revisiting Semi-Supervised Learning with Graph Embeddings}. In \bibinfo{booktitle}{\emph{{ICML}}}. \bibinfo{pages}{40--48}.
\newblock


\bibitem[Yang et~al\mbox{.}(2022)]%
        {DBLP:conf/icdm/YangHWLLW22}
\bibfield{author}{\bibinfo{person}{Zhe{-}Rui Yang}, \bibinfo{person}{Zhen{-}Yu He}, \bibinfo{person}{Chang{-}Dong Wang}, \bibinfo{person}{Pei{-}Yuan Lai}, \bibinfo{person}{De{-}Zhang Liao}, {and} \bibinfo{person}{Zhong{-}Zheng Wang}.} \bibinfo{year}{2022}\natexlab{}.
\newblock \showarticletitle{A Bi-directional Recommender System for Online Recruitment}. In \bibinfo{booktitle}{\emph{{ICDM}}}. \bibinfo{pages}{628--637}.
\newblock


\bibitem[Yu and Zhu(2020)]%
        {DBLP:journals/corr/abs-2003-05689}
\bibfield{author}{\bibinfo{person}{Tong Yu} {and} \bibinfo{person}{Hong Zhu}.} \bibinfo{year}{2020}\natexlab{}.
\newblock \showarticletitle{Hyper-Parameter Optimization: {A} Review of Algorithms and Applications}.
\newblock \bibinfo{journal}{\emph{CoRR}}  \bibinfo{volume}{abs/2003.05689} (\bibinfo{year}{2020}).
\newblock


\bibitem[Zaken et~al\mbox{.}(2022)]%
        {DBLP:conf/acl/ZakenGR22}
\bibfield{author}{\bibinfo{person}{Elad~Ben Zaken}, \bibinfo{person}{Yoav Goldberg}, {and} \bibinfo{person}{Shauli Ravfogel}.} \bibinfo{year}{2022}\natexlab{}.
\newblock \showarticletitle{BitFit: Simple Parameter-efficient Fine-tuning for Transformer-based Masked Language-models}. In \bibinfo{booktitle}{\emph{{ACL} {(2)}}}. \bibinfo{pages}{1--9}.
\newblock


\bibitem[Zeng and Lee(2023)]%
        {DBLP:journals/corr/abs-2310-17513}
\bibfield{author}{\bibinfo{person}{Yuchen Zeng} {and} \bibinfo{person}{Kangwook Lee}.} \bibinfo{year}{2023}\natexlab{}.
\newblock \showarticletitle{The Expressive Power of Low-Rank Adaptation}.
\newblock \bibinfo{journal}{\emph{CoRR}}  \bibinfo{volume}{abs/2310.17513} (\bibinfo{year}{2023}).
\newblock


\bibitem[Zhang et~al\mbox{.}(2021)]%
        {DBLP:conf/nips/ZhangWYWY21}
\bibfield{author}{\bibinfo{person}{Hengrui Zhang}, \bibinfo{person}{Qitian Wu}, \bibinfo{person}{Junchi Yan}, \bibinfo{person}{David Wipf}, {and} \bibinfo{person}{Philip~S. Yu}.} \bibinfo{year}{2021}\natexlab{}.
\newblock \showarticletitle{From Canonical Correlation Analysis to Self-supervised Graph Neural Networks}. In \bibinfo{booktitle}{\emph{NeurIPS}}. \bibinfo{pages}{76--89}.
\newblock


\bibitem[Zhang et~al\mbox{.}(2022a)]%
        {DBLP:conf/ijcai/ZhangXHRB22}
\bibfield{author}{\bibinfo{person}{Jiying Zhang}, \bibinfo{person}{Xi Xiao}, \bibinfo{person}{Long{-}Kai Huang}, \bibinfo{person}{Yu Rong}, {and} \bibinfo{person}{Yatao Bian}.} \bibinfo{year}{2022}\natexlab{a}.
\newblock \showarticletitle{Fine-Tuning Graph Neural Networks via Graph Topology Induced Optimal Transport}. In \bibinfo{booktitle}{\emph{{IJCAI}}}. \bibinfo{pages}{3730--3736}.
\newblock


\bibitem[Zhang et~al\mbox{.}(2023)]%
        {DBLP:conf/iccv/ZhangRA23}
\bibfield{author}{\bibinfo{person}{Lvmin Zhang}, \bibinfo{person}{Anyi Rao}, {and} \bibinfo{person}{Maneesh Agrawala}.} \bibinfo{year}{2023}\natexlab{}.
\newblock \showarticletitle{Adding Conditional Control to Text-to-Image Diffusion Models}. In \bibinfo{booktitle}{\emph{{ICCV}}}. \bibinfo{publisher}{{IEEE}}, \bibinfo{pages}{3813--3824}.
\newblock


\bibitem[Zhang et~al\mbox{.}(2020)]%
        {DBLP:journals/corr/abs-2010-16103}
\bibfield{author}{\bibinfo{person}{Muhan Zhang}, \bibinfo{person}{Pan Li}, \bibinfo{person}{Yinglong Xia}, \bibinfo{person}{Kai Wang}, {and} \bibinfo{person}{Long Jin}.} \bibinfo{year}{2020}\natexlab{}.
\newblock \showarticletitle{Revisiting Graph Neural Networks for Link Prediction}.
\newblock \bibinfo{journal}{\emph{CoRR}}  \bibinfo{volume}{abs/2010.16103} (\bibinfo{year}{2020}).
\newblock


\bibitem[Zhang et~al\mbox{.}(2022b)]%
        {DBLP:conf/kdd/ZhangZSKK22}
\bibfield{author}{\bibinfo{person}{Yifei Zhang}, \bibinfo{person}{Hao Zhu}, \bibinfo{person}{Zixing Song}, \bibinfo{person}{Piotr Koniusz}, {and} \bibinfo{person}{Irwin King}.} \bibinfo{year}{2022}\natexlab{b}.
\newblock \showarticletitle{{COSTA:} Covariance-Preserving Feature Augmentation for Graph Contrastive Learning}. In \bibinfo{booktitle}{\emph{{KDD}}}. \bibinfo{pages}{2524--2534}.
\newblock


\bibitem[Zhao et~al\mbox{.}(2024)]%
        {DBLP:journals/corr/abs-2402-09834}
\bibfield{author}{\bibinfo{person}{Haihong Zhao}, \bibinfo{person}{Aochuan Chen}, \bibinfo{person}{Xiangguo Sun}, \bibinfo{person}{Hong Cheng}, {and} \bibinfo{person}{Jia Li}.} \bibinfo{year}{2024}\natexlab{}.
\newblock \showarticletitle{All in One and One for All: {A} Simple yet Effective Method towards Cross-domain Graph Pretraining}.
\newblock \bibinfo{journal}{\emph{CoRR}}  \bibinfo{volume}{abs/2402.09834} (\bibinfo{year}{2024}).
\newblock


\bibitem[Zhou et~al\mbox{.}(2021)]%
        {DBLP:conf/sigir/ZhouTHZW21}
\bibfield{author}{\bibinfo{person}{Huachi Zhou}, \bibinfo{person}{Qiaoyu Tan}, \bibinfo{person}{Xiao Huang}, \bibinfo{person}{Kaixiong Zhou}, {and} \bibinfo{person}{Xiaoling Wang}.} \bibinfo{year}{2021}\natexlab{}.
\newblock \showarticletitle{Temporal Augmented Graph Neural Networks for Session-Based Recommendations}. In \bibinfo{booktitle}{\emph{{SIGIR}}}. \bibinfo{pages}{1798--1802}.
\newblock


\bibitem[Zhou et~al\mbox{.}(2020)]%
        {DBLP:conf/ijcai/ZhouSHZZH20}
\bibfield{author}{\bibinfo{person}{Kaixiong Zhou}, \bibinfo{person}{Qingquan Song}, \bibinfo{person}{Xiao Huang}, \bibinfo{person}{Daochen Zha}, \bibinfo{person}{Na Zou}, {and} \bibinfo{person}{Xia Hu}.} \bibinfo{year}{2020}\natexlab{}.
\newblock \showarticletitle{Multi-Channel Graph Neural Networks}. In \bibinfo{booktitle}{\emph{{IJCAI}}}. \bibinfo{pages}{1352--1358}.
\newblock


\bibitem[Zhu et~al\mbox{.}(2023)]%
        {DBLP:journals/corr/abs-2310-07365}
\bibfield{author}{\bibinfo{person}{Yun Zhu}, \bibinfo{person}{Yaoke Wang}, \bibinfo{person}{Haizhou Shi}, \bibinfo{person}{Zhenshuo Zhang}, {and} \bibinfo{person}{Siliang Tang}.} \bibinfo{year}{2023}\natexlab{}.
\newblock \showarticletitle{GraphControl: Adding Conditional Control to Universal Graph Pre-trained Models for Graph Domain Transfer Learning}.
\newblock \bibinfo{journal}{\emph{CoRR}}  \bibinfo{volume}{abs/2310.07365} (\bibinfo{year}{2023}).
\newblock


\bibitem[Zhu et~al\mbox{.}(2020)]%
        {DBLP:journals/corr/abs-2006-04131}
\bibfield{author}{\bibinfo{person}{Yanqiao Zhu}, \bibinfo{person}{Yichen Xu}, \bibinfo{person}{Feng Yu}, \bibinfo{person}{Qiang Liu}, \bibinfo{person}{Shu Wu}, {and} \bibinfo{person}{Liang Wang}.} \bibinfo{year}{2020}\natexlab{}.
\newblock \showarticletitle{Deep Graph Contrastive Representation Learning}.
\newblock \bibinfo{journal}{\emph{CoRR}}  \bibinfo{volume}{abs/2006.04131} (\bibinfo{year}{2020}).
\newblock


\bibitem[Zhuang et~al\mbox{.}(2021)]%
        {DBLP:journals/pieee/ZhuangQDXZZXH21}
\bibfield{author}{\bibinfo{person}{Fuzhen Zhuang}, \bibinfo{person}{Zhiyuan Qi}, \bibinfo{person}{Keyu Duan}, \bibinfo{person}{Dongbo Xi}, \bibinfo{person}{Yongchun Zhu}, \bibinfo{person}{Hengshu Zhu}, \bibinfo{person}{Hui Xiong}, {and} \bibinfo{person}{Qing He}.} \bibinfo{year}{2021}\natexlab{}.
\newblock \showarticletitle{A Comprehensive Survey on Transfer Learning}.
\newblock \bibinfo{journal}{\emph{Proc. {IEEE}}} \bibinfo{volume}{109}, \bibinfo{number}{1} (\bibinfo{year}{2021}), \bibinfo{pages}{43--76}.
\newblock


\end{thebibliography}

\appendix
\section{Theoretical Analysis}
\label{appendix:theory}
\subsection{Notations}
We define $[N] \coloneqq \{1,2,\dots,N\}$. The SVD decomposition of matrix $\boldsymbol{W}$ is given as $\boldsymbol{W}=\boldsymbol{UDV}^T$, and $\sigma _i(\boldsymbol{W}) =D_{i,i}$. The best rank-$r$ approximation (in the Frobenius norm or the 2-norm) of $\boldsymbol{W}$ is $\sum_{i=1}^r\sigma_i(\boldsymbol{W})\boldsymbol{u}_i\boldsymbol{v}_{i}^{T}$, where $\boldsymbol{u}_i$ and $\boldsymbol{v}_i$ are the $i$-th column of $\boldsymbol{U}$ and $\boldsymbol{V}$, respectively~\cite{DBLP:journals/corr/abs-2310-17513}. Following~\cite{DBLP:journals/corr/abs-2310-17513}, we define the best rank-$r$ approximation as $LR_r( \boldsymbol{W} )$. When $r\geqslant rank( \boldsymbol{W} )$, it is obvious that $LR_r( \boldsymbol{W} ) = \boldsymbol{W}$. Considering the differences in architectures among various GNNs, for the sake of analytical simplicity, we consider the following general GNN architecture:
\begin{align}
    \mathrm{Propagate}: \bar{\boldsymbol{H}}^l=&\boldsymbol{H}^{l-1}\boldsymbol{P}, \\
    \mathrm{Transform}: \boldsymbol{H}^l=&\mathrm{Re}LU( \boldsymbol{W}^l\bar{\boldsymbol{H}}^l+\boldsymbol{B}^l ), 
\end{align}
where $\boldsymbol{B}^l=\boldsymbol{b}^l\mathbf{1},\mathbf{1}=[1,1,\dots,1] \in \mathbb{R}^{1\times N}$, and $\boldsymbol{P}$ represents the message transformation matrix associated with the adjacency matrix. For simplicity, assume that $(\boldsymbol{W}^l)_{l=1}^{L} \in \mathbb{R}^{D \times D}$ and $(\boldsymbol{b}^l)_{l=1}^{L} \in \mathbb{R}^{D \times 1}$. 

We define an $L$-layer width-$D$ graph neural network as follows: 
$GNN_{L, D}(\cdot; (\boldsymbol{W}^l)_{l=1}^{L}, (\boldsymbol{b}^l)_{l=1}^{L}) \coloneqq \mathrm{ReLU}(\boldsymbol{W}^L \mathrm{ReLU}(\boldsymbol{W}^{L-1} \mathrm{ReLU}(\ldots) \boldsymbol{P} + \boldsymbol{B}^{L-1}) \boldsymbol{P} + \boldsymbol{B}^L)$. The target GNN $\overline{g}$, frozen GNN $g_0$, and adapted GNN $g$ are defined as follows: 
\begin{align}
    \overline{g} =& GNN_{\overline{L}, D}(\cdot; (\overline{\boldsymbol{W}}^l)_{l=1}^{\overline{L}}, (\overline{\boldsymbol{b}}^l)_{l=1}^{\overline{L}}), \\
    g_0 =& GNN_{L, D}(\cdot; (\boldsymbol{W}^l)_{l=1}^{L}, (\boldsymbol{b}^l)_{l=1}^{L}), \\
    g =& GNN_{L, D}(\cdot; (\boldsymbol{W}^l + \bigtriangleup \boldsymbol{W}^l)_{l=1}^{L}, (\hat{\boldsymbol{b}}^l)_{l=1}^{L}),
\end{align}
where $\overline{L} \leqslant L$. 

We define an $L$-layer width-$D$ fully connected neural network (FNN) as follows: $FNN_{L, D}(\cdot; (\boldsymbol{W}_l)_{l=1}^{L}, (\boldsymbol{b}_l)_{l=1}^{L}) \coloneqq \\ \mathrm{ReLU}(\boldsymbol{W}_L \mathrm{ReLU}(\boldsymbol{W}_{L-1} \mathrm{ReLU}(\ldots) + \boldsymbol{b}_{L-1}) + \boldsymbol{b}_L)$, where $(\boldsymbol{W}_l)_{l=1}^{L} \in \mathbb{R}^{D \times D}$ are weight matrices, and $(\boldsymbol{b}_l)_{l=1}^{L} \in \mathbb{R}^{D}$ are bias vectors. The target FNN $\overline{f}$, frozen FNN $f_0$, and adapted FNN $f$ are defined as: 
\begin{align}
    \overline{f} = & FNN_{\overline{L},D}(\cdot; (\overline{\boldsymbol{W}}_l)_{l=1}^{\overline{L}}, (\overline{\boldsymbol{b}}_l)_{l=1}^{\overline{L}}), \\
    f_0 = & FNN_{L,D}(\cdot; (\boldsymbol{W}_l)_{l=1}^{L}, (\boldsymbol{b}_l)_{l=1}^{L}), \\
    f = & FNN_{L,D}(\cdot; (\boldsymbol{W}_l + \bigtriangleup \boldsymbol{W}_l)_{l=1}^{L}, (\hat{\boldsymbol{b}}_l)_{l=1}^{L}),
\end{align}
where $\overline{L} \leqslant L$.

In addition, we define the partition $\mathcal{P} = \{ P_1, \cdots, P_{\overline{L}} \} = \{ \{ 1, \cdots, \\ M \}, \{ M+1, \cdots, 2M \}, \cdots, \{ ( \overline{L}-1 ) M+1, \cdots, L \} \}$, where $M = \lfloor L/\overline{L} \rfloor$.

\subsection{Expressive Power of Fully Connected Neural Networks with LoRA}
Before presenting the theoretical analysis of the expressive power of graph neural networks with LoRA, we introduce the relevant lemmas from~\cite{DBLP:journals/corr/abs-2310-17513} that discuss the expressive power of fully connected neural networks with LoRA.
\begin{assumption}
\label{assumption1} For a fixed rank $R\in [D] $, the weight matrices of the frozen model $( \boldsymbol{W}_l ) _{l=1}^{L}$ and matrices $( \prod_{l\in P_i}{\boldsymbol{W}_l} ) +  LR_r( \overline{\boldsymbol{W}}_i -  \prod_{l\in P_i}{\boldsymbol{W}_l} ) $ are non-singular for all $r\leqslant R( M-1 ) $ and $i \in [\overline{L}]$.
\end{assumption}

\begin{lemma}
\label{lemma1}
Let $( \overline{\boldsymbol{W}}_i ) _{i=1}^{\overline{L}}, ( \boldsymbol{W}_l ) _{l=1}^{L}\in \mathbb{R}^{D\times D}$ matrices whose elements are drawn independently from arbitrary continuous distributions. Then, with probability 1, Assumption~\ref{assumption1} holds $\forall R\in [D] $.
\end{lemma}

\begin{lemma}
\label{lemma2}
Under Assumption~\ref{assumption1}, if rank $R\geqslant \lceil \max_{i\in [ \overline{L} ]} rank( \overline{\boldsymbol{W}_i}-\prod_{l\in P_i}{\boldsymbol{W}_l} ) /M \rceil $, then there exists rank-$R$ or lower matrices $\bigtriangleup \boldsymbol{W}_1,\cdots ,\\ \bigtriangleup \boldsymbol{W}_L\in \mathbb{R}^{D\times D}$ and bias vectors $\hat{\boldsymbol{b}}_1,\cdots ,\hat{\boldsymbol{b}}_L\in \mathbb{R}^D$ such that the low-rank adapted model f can exactly approximate the target model $\overline{f}$, i.e., $f(\boldsymbol{x})=\overline{f}(\boldsymbol{x})$, $\forall \boldsymbol{x}\in \mathcal{X} $, where $\mathcal{X}$ is the input space.
\end{lemma}

\begin{lemma}
\label{lemma3}
Define $E_i=\sigma _{RM+1}( \overline{\boldsymbol{W}_i}-\prod_{l\in P_i}{\boldsymbol{W}_l} ) $, and $\xi = \\ \max ( \max_{i\in [ \overline{L} ]} ( \sqrt{\| \Sigma \| _F}\prod_{j=1}^i{\| \overline{\boldsymbol{W}}_j \| _F + \sum_{j=1}^i{\prod_{k=j+1}^{i-1}{\| \overline{\boldsymbol{W}}_k \| _F\| \overline{\boldsymbol{b}}_j \| _2}}} ) ,\\ \sqrt{\| \Sigma \| _F} ) $. Under Assumption~\ref{assumption1}, there exists rank-$R$ or lower matrices $\bigtriangleup ( \boldsymbol{W}_l ) _{l=1}^{L}\in \mathbb{R}^{D\times D}$ and bias vectors $( \hat{\boldsymbol{b}}_l ) _{l=1}^{L}\in \mathbb{R}^D$ for any input $\boldsymbol{x}\in \mathcal{X}$ with $\Sigma =\mathbb{E} \boldsymbol{xx}^T$, such that
\begin{align}
    \mathbb{E} \| f( \boldsymbol{x} ) -\overline{f}( \boldsymbol{x} ) \| _2\leqslant \xi \sum_{i=1}^{\overline{L}}{\max_{k\in [ \overline{L} ]}}( \| \overline{\boldsymbol{W}}_k \| _F+E_k ) ^{\overline{L}-i}E_i.
\end{align}
\end{lemma}

The proofs of these lemmas can be found in~\cite{DBLP:journals/corr/abs-2310-17513}, specifically within the proofs of Lemma 3, Theorem 3, and Theorem 5.

\subsection{Expressive Power of Graph Neural Networks with LoRA}
\begin{assumption}
\label{assumption2} For a fixed rank $R\in [D] $, the weight matrices of the frozen model $( \boldsymbol{W}^l )_{l=1}^{L}$ and matrices $( \prod_{l\in P_i}{\boldsymbol{W}^l} ) +  LR_r( \overline{\boldsymbol{W}}^i -  \prod_{l\in P_i}{\boldsymbol{W}^l} ) $ are non-singular for all $r\leqslant R( M-1 ) $ and $i \in [\overline{L}]$.
\end{assumption}

\begin{theorem}
\label{theorem1}
Under Assumption~\ref{assumption2}, if rank $R\geqslant \lceil \max_{i\in [ \overline{L} ]} rank( \overline{\boldsymbol{W}}^i-\prod_{l\in P_i}{\boldsymbol{W}^l} ) /M \rceil $, then there exists rank-$R$ or lower matrices $\bigtriangleup \boldsymbol{W}^1,\cdots ,\\ \bigtriangleup \boldsymbol{W}^L\in \mathbb{R}^{D\times D}$ and bias vectors $\hat{\boldsymbol{b}}^1,\cdots ,\hat{\boldsymbol{b}}^L\in \mathbb{R}^{D\times 1}$ such that the low-rank adapted model $g$ can exactly approximate the target model $\overline{g}$, i.e., $g(\boldsymbol{X}) =\overline{g}(\boldsymbol{X})$, $\forall \boldsymbol{X}\in \mathcal{X}^{'} $, where $\mathcal{X}^{'}$ is the input space.
\end{theorem}
\begin{proof}
Revisiting the GNN expression: $\boldsymbol{H}^l = \mathrm{ReLU}(\boldsymbol{W}^l\boldsymbol{H}^{l-1}\boldsymbol{P} + \boldsymbol{B}^l)$, we observe that the GNN expression resembles the FNN expression, but with two key differences: (1) the FNN input is a single sample, while the GNN input includes all samples; (2) the GNN requires message passing at each layer: $\bar{\boldsymbol{H}}^l = \boldsymbol{H}^{l-1}\boldsymbol{P}$. We set $\boldsymbol{W}_l=\boldsymbol{W}^l$ for $l\in [L] $ and $\overline{\boldsymbol{W}}_i=\overline{\boldsymbol{W}}^i$ for $i\in [ \overline{L} ] $. Under Assumption~\ref{assumption2}, $( \boldsymbol{W}_l ) _{l=1}^{L}$ and $( \overline{\boldsymbol{W}}_i ) _{i=1}^{\overline{L}}$ satisfy Assumption~\ref{assumption1}. Despite the aforementioned differences, the proof process of Lemma~\ref{lemma2}~\cite{DBLP:journals/corr/abs-2310-17513} demonstrates that these differences do not affect the proof of the conclusion. Consequently, we can easily deduce that Theorem~\ref{theorem1} holds.
\end{proof}

It is noteworthy that, based on Lemma~\ref{lemma1}, Assumption~\ref{assumption2} holds in most cases. Revisiting the conditions for Theorem~\ref{theorem1} to hold, if $R$ fails to meet the condition, can we offer an approximate upper bound on the difference between $g(\boldsymbol{X})$ and $\overline{g}(\boldsymbol{X})$? The answer is affirmative.

\begin{theorem}
\label{theorem2}
Define $E^i=\sigma _{RM+1}( \overline{\boldsymbol{W}}^i-\prod_{l\in P_i}{\boldsymbol{W}^l} ) $, and $\xi^{'} = \max ( \max_{i\in [ \overline{L} ]} ( \mathbb{E} \| \boldsymbol{X} \| _2\prod_{j=1}^i{\| \overline{\boldsymbol{W}}^j} \| _2\| \boldsymbol{P} \| _{2}^{i} + \sum_{j=1}^i{\prod_{k=j+1}^{i-1}{\| \overline{\boldsymbol{W}}^k \| _2}} \\ \| \overline{\boldsymbol{B}}^j \| _2\| \boldsymbol{P} \| _{2}^{i-j-1} ) ,\mathbb{E} \| \boldsymbol{X} \| _2 ) $. Under Assumption~\ref{assumption2}, there exists rank-$R$ or lower matrices $\bigtriangleup ( \boldsymbol{W}^l ) _{l=1}^{L}\in \mathbb{R}^{D\times D}$ and bias vectors $( \hat{\boldsymbol{b}}^l )_{l=1}^{L}\in \mathbb{R}^{D\times 1}$ for any input $\boldsymbol{X}\in \mathcal{X}^{'}$, such that
\begin{align}
\small
    \mathbb{E} \| g( \boldsymbol{X} ) -\overline{g}( \boldsymbol{X} ) \| _2\leqslant \xi^{'} \sum_{i=1}^{\overline{L}}{\max_{k\in [ \overline{L} ]}}( \| \overline{\boldsymbol{W}}^k \| _2+E^k ) ^{\overline{L}-i}E^i\| \boldsymbol{P} \| _{2}^{\overline{L}-i+1}.
\end{align}
\end{theorem}
\begin{proof}
Similarly, we set $\boldsymbol{W}_l=\boldsymbol{W}^l$ for $l\in [L] $ and $\overline{\boldsymbol{W}}_i=\overline{\boldsymbol{W}}^i$ for $i\in [ \overline{L} ] $. Under Assumption~\ref{assumption2}, it follows that $( \boldsymbol{W}_l ) _{l=1}^{L}$ and $( \overline{\boldsymbol{W}}_i ) _{i=1}^{\overline{L}}$ satisfy Assumption~\ref{assumption1}. Following the proof of Lemma~\ref{lemma3}~\cite{DBLP:journals/corr/abs-2310-17513} and substituting the GNN expression for the FNN expression in the proof of Lemma~\ref{lemma3}~\cite{DBLP:journals/corr/abs-2310-17513}, it is easy to deduce that Theorem~\ref{theorem2} holds.
\end{proof}

\section{Datasets}
\label{appendix:data}
\subsection{Dataset Statistics}
Detailed statistics of the datasets are provided in \tablename~\ref{tab:stat_detail}.

\begin{table}[h]
\caption{Statistics of datasets.}
\label{tab:stat_detail}
\begin{tabular}{@{}c|cccc@{}}
\toprule
Dataset  & \#Nodes & \#Edges & \#Features & \#Classes \\ \midrule
PubMed   & 19,717   & 88,651   & 500        & 3         \\
CiteSeer & 3,327    & 9,228    & 3,703       & 6         \\
Cora     & 2,708    & 10,556   & 1,433       & 7         \\
Photo    & 7,650    & 238,163  & 745        & 8         \\
Computer & 13,752   & 491,722  & 767        & 10        \\
Reddit   & 232,965 & 114,615,892  & 602   & 41        \\
ogbn-arxiv   & 169,343 & 1,166,243  & 128   & 40        \\
ogbn-products   & 2,449,029 &  61,859,140  & 100   & 47        \\
Squirrel   & 5201 & 217073  & 2089   & 5        \\
Chameleon   & 2277 & 36101  & 2325   & 5        \\ \bottomrule
\end{tabular}
\end{table}

\subsection{Dataset Descriptions}
\begin{itemize}
  \item \textbf{PubMed, CiteSeer, and Cora}. PubMed, CiteSeer, and Cora are three standard citation network benchmark datasets. In these datasets, the nodes correspond to research papers, and the edges represent the citations between papers. The node features are derived from the bag-of-words representation of the papers, while the node labels indicate the academic topics of the papers.
  \item \textbf{Photo and Computer}. The Photo and Computer datasets are segments of the Amazon co-purchase graph~\cite{DBLP:conf/sigir/McAuleyTSH15}. In these datasets, the nodes represent products, while the edges represent the frequent co-purchasing relationship between two products. The node features are derived from the bag-of-words representation of product reviews, and the labels indicate the categories of the products.
  \item \textbf{Reddit}. The Reddit dataset is constructed from the Reddit online discussion forum. In this dataset, nodes represent posts, while edges indicate instances where the same user has commented on both connecting posts. The node features are derived from the GloVe CommonCrawl word representation of posts~\cite{DBLP:conf/emnlp/PenningtonSM14}, while labels indicate the community to which the post belongs.
  \item \textbf{ogbn}. The ogbn-arxiv dataset is a citation network of arXiv papers, where each node represents a paper and edges denote citations. Features are derived by averaging the word embeddings from the title and abstract of each paper, and labels correspond to the subject areas of arXiv papers. The ogbn-products dataset represents an Amazon product co-purchasing network, where each node represents a product, and edges indicate products purchased together. Node features are derived from a bag-of-words representation of the product descriptions, and labels correspond to the product categories.
  \item \textbf{Squirrel and Chameleon}. Squirrel and Chameleon are two page-page networks in Wikipedia, where nodes represent web pages and edges represent mutual links between pages. Node features correspond to several informative nouns in the pages, and labels correspond to the average monthly traffic of the web pages.
\end{itemize}

\begin{table*}[!t]
\caption{Performance on heterogeneous graphs.}
\label{tab:hete}
\setlength{\tabcolsep}{1mm}
\resizebox{\linewidth}{!}{
\begin{tabular}{c|cccccccccccccc}
\hline
Method    & GCN      & GAT      & COSTA    & CCA-SSG  & HomoGCL  & GRACE    & GPPT     & GPF      & GraphPrompt & ProG     & GTOT     & AdapterGNN & GraphControl & GraphLoRA \\ \hline
Squirrel  & 37.3±0.6 & 33.3±1.3 & 41.6±1.3 & 33.2±1.0 & 32.5±1.2 & 28.5±0.8 & 32.7±0.5 & 31.8±1.1 & 27.6±0.8    & 30.8±1.3 & 40.6±1.7 & 35.7±1.2   & 41.3±1.5     & 35.2±1.0  \\
Chameleon & 58.4±1.0 & 53.6±2.0 & 58.2±1.9 & 47.3±2.2 & 46.3±1.8 & 46.7±2.2 & 52.7±1.0 & 52.2±1.1 & 40.6±2.0    & 49.0±2.4 & 54.5±2.7 & 54.9±2.1   & 55.4±1.9     & 57.5±2.9  \\ \hline
\end{tabular}}
\end{table*}

\begin{table*}[!t]
\caption{Performance with different pretraining methods. The notations "-PM," "-CS," "-C," "-P," and "-Com" represent the pre-training datasets PubMed, CiteSeer, Cora, Photo, and Computer, respectively. GraphLoRA(CS) and GraphLoRA(HG) represent the application of the CCA-SSG and HomoGCL pretraining methods, respectively.}
\label{tab:diff_pretrain}
\resizebox{\textwidth}{!}{
\begin{tabular}{@{}c|cccccccccc@{}}
\toprule
\multirow{2}{*}{Method} & \multicolumn{2}{c}{PubMed} & \multicolumn{2}{c}{CiteSeer} & \multicolumn{2}{c}{Cora} & \multicolumn{2}{c}{Photo} & \multicolumn{2}{c}{Computer} \\ \cmidrule(l){2-3} \cmidrule(l){4-5} \cmidrule(l){6-7} \cmidrule(l){8-9} \cmidrule(l){10-11}
& public & 10-shot & public & 10-shot & public & 10-shot & public & 10-shot & public & 10-shot \\ \midrule
CCA-SSG-PM  & 80.12±1.13 & 75.03±1.28 & 59.48±2.58 & 52.37±3.42 & 73.30±0.12 & 63.68±2.08 & 92.13±0.24 & 83.27±0.99 & 87.15±0.35 & 75.14±0.85 \\
CCA-SSG-CS  & 75.12±0.79 & 55.95±2.97 & 72.26±1.14 & 31.95±0.40 & 73.90±1.40 & 31.09±0.98 & 92.34±0.32 & 85.90±0.37 & 86.90±0.50 & 75.53±0.59 \\
CCA-SSG-C   & 69.52±1.37 & 62.07±3.77 & 56.22±2.47 & 45.49±1.97 & \textbf{82.56±0.57} & 76.78±1.78 & 92.10±0.33 & 85.23±0.69 & 87.85±0.41 & 74.92±1.08 \\
CCA-SSG-P   & 72.04±1.41 & 69.05±2.96 & 57.74±1.66 & 51.12±4.11 & 68.48±2.18 & 57.37±1.91 & 92.82±0.22 & 85.25±0.74 & 87.97±0.54 & 71.54±0.88 \\
CCA-SSG-Com & 72.42±1.97 & 69.22±2.96 & 63.82±1.15 & 56.34±2.44 & 63.82±1.75 & 57.54±3.48 & 92.74±0.29 & 82.95±1.98 & \uline{88.08±0.35} & 75.32±1.76 \\ \hline
HomoGCL-PM  & 79.46±0.11 & \textbf{77.22±0.11} & 66.82±0.08 & 55.30±1.08 & 75.44±0.05 & 63.36±0.18 & 92.34±0.26 & 81.50±0.27 & 87.76±0.20 & 76.24±0.04 \\
HomoGCL-CS  & 78.16±0.15 & 70.04±0.15 & 72.00±0.10 & 68.54±0.09 & 76.96±0.05 & 67.02±0.16 & 92.36±0.24 & 81.64±0.38 & 87.30±0.49 & 75.80±0.03 \\
HomoGCL-C   & 74.40±0.00 & 67.30±0.00 & 64.26±0.09 & 54.60±0.00 & 81.50±0.00 & 76.32±0.40 & 92.22±0.25 & 83.95±0.16 & 87.09±0.34 & 76.00±0.01 \\
HomoGCL-P   & 76.04±0.59 & 69.14±0.47 & 70.20±0.70 & 57.80±0.63 & 74.28±0.51 & 59.94±1.69 & 92.43±0.08 & 83.75±0.23 & 86.19±0.31 & 75.24±0.07 \\
HomoGCL-Com & 78.00±1.02 & 73.28±1.85 & 64.80±1.00 & 62.64±1.79 & 70.98±2.08 & 63.44±1.99 & 92.71±0.19 & 83.02±0.71 & \textbf{88.34±0.18} & 74.83±1.12 \\ \hline
GraphLoRA(CS)-PM  & \uline{80.52±0.38} & 76.78±1.11 & 73.58±0.37 & 72.34±1.45 & \uline{82.20±0.44} & 76.96±1.37 & \textbf{92.95±0.03} & 87.99±1.05 & 87.78±0.17 & 75.79±0.32 \\
GraphLoRA(CS)-CS  & \textbf{80.60±0.41} & 76.96±1.66 & 71.78±1.34 & 71.32±1.38 & 80.32±0.62 & 75.76±1.40 & 92.41±0.29 & 87.81±1.22 & 87.43±0.41 & 75.27±0.75 \\
GraphLoRA(CS)-C   & 79.26±1.00 & 75.46±1.24 & 70.98±1.58 & 71.96±0.73 & 80.14±0.57 & 75.06±1.50 & 92.50±0.53 & 87.62±0.88 & 87.51±0.27 & 74.73±0.78 \\
GraphLoRA(CS)-P   & 76.22±2.08 & 73.54±1.02 & 62.40±3.21 & 59.14±4.92 & 79.66±0.94 & 74.26±1.27 & 92.05±0.23 & 88.00±0.85 & 87.02±0.60 & 74.81±0.89 \\
GraphLoRA(CS)-Com & 76.32±1.37 & 73.10±0.58 & 61.40±4.4  & 51.12±7.25 & 78.74±0.88 & 75.04±1.75 & 91.93±0.73 & \textbf{88.70±0.73} & 86.34±0.53 & 74.48±0.96 \\ \hline
GraphLoRA(HG)-PM  & 80.16±0.43 & \uline{77.12±0.90} & \textbf{74.22±0.50} & \uline{74.18±0.58} & 82.16±0.24 & \textbf{77.64±0.66} & \uline{92.91±0.02} & \uline{88.23±0.79} & 88.03±0.21 & \textbf{76.50±0.38} \\
GraphLoRA(HG)-CS  & 79.76±0.18 & \uline{77.12±0.79} & \uline{73.72±0.58} & \textbf{74.58±0.40} & 81.70±0.25 & 77.18±0.99 & 92.67±0.31 & 87.95±0.78 & 87.62±0.59 & 76.24±0.24 \\
GraphLoRA(HG)-C   & 79.44±0.38 & 76.58±1.05 & 73.12±0.74 & 73.26±0.44 & 81.42±0.23 & \uline{77.50±0.62} & 92.69±0.31 & 87.71±1.04 & 87.89±0.42 & 76.25±0.31 \\
GraphLoRA(HG)-P   & 76.68±1.62 & 73.00±0.96 & 65.60±1.55 & 39.66±5.48 & 73.84±1.48 & 72.70±1.33 & 90.93±0.85 & 88.19±0.72 & 86.85±0.48 & 76.26±0.26 \\
GraphLoRA(HG)-Com & 75.38±1.60 & 74.62±3.07 & 67.84±1.53 & 35.90±3.26 & 76.26±0.88 & 75.30±1.33 & 91.52±0.70 & 87.87±1.06 & 87.11±0.42 & \uline{76.27±0.22} \\ \bottomrule
\end{tabular}}
\end{table*}

\begin{table*}[!t]
\caption{Comparison of experimental results in the 5-shot setting. The notations "-PM," "-CS," "-C," "-P," and "-Com" represent the pre-training datasets PubMed, CiteSeer, Cora, Photo, and Computer, respectively. The best experimental results are highlighted in bold, while the second-best results are underscored with a underline.}
\label{tab:5shot}
\begin{tabular}{@{}cc|lllll@{}}
\toprule
\multicolumn{2}{c|}{Method}        & \multicolumn{1}{c}{PubMed} & \multicolumn{1}{c}{CiteSeer} & \multicolumn{1}{c}{Cora} & \multicolumn{1}{c}{Photo} & \multicolumn{1}{c}{Computer} \\ \midrule
\multicolumn{1}{c|}{\multirow{10}{*}{non-transfer}} & GCN           & 70.00±0.67                 & 61.92±2.12                   & 67.78±0.61               & 86.10±0.65                & 68.79±2.36                   \\
\multicolumn{1}{c|}{} & GAT           & 69.98±0.34                 & 62.42±1.43                   & 68.56±1.41               & 86.86±0.71                & 73.23±1.38                   \\
\multicolumn{1}{c|}{} & GRACE         & 69.38±0.08                 & 64.08±0.13                   & 72.82±0.44               & \uline{87.53±0.01}          & 71.64±0.13                   \\
\multicolumn{1}{c|}{} & COSTA         & 71.28±0.91                 & 67.26±2.44                   & 74.58±0.88               & 83.30±0.80                & 67.39±0.59                   \\
\multicolumn{1}{c|}{} & CCA-SSG       & 72.02±0.80                 & 67.78±1.60                   & 75.08±0.88               & 86.82±1.03                & 76.01±1.11                   \\
\multicolumn{1}{c|}{} & HomoGCL       & 70.30±0.00                 & 67.50±0.10                   & 75.44±0.00               & 83.73±0.18                & \uline{77.64±0.48}             \\
\multicolumn{1}{c|}{} & GPPT          & 69.68±0.37                 & 60.62±0.66                   & 66.80±0.53               & 85.15±0.44                & \textbf{78.16±0.68}          \\ 
\multicolumn{1}{c|}{} & GPF         & 69.46±1.75 & 62.68±1.63 & 71.76±0.88 & 87.19±0.88 & 75.30±0.49 \\
\multicolumn{1}{c|}{} & GraphPrompt & 75.23±0.93 & 69.71±1.06 & 79.90±0.74 & 86.35±0.41 & 72.43±0.27 \\
\multicolumn{1}{c|}{} & ProG        & 70.08±0.77 & 64.21±0.73 & 71.74±1.13 & 87.76±0.75 & 73.69±1.14                   \\ \midrule
\multicolumn{1}{c|}{\multirow{25}{*}{transfer}} & GRACE-P       & 69.36±0.17                 & 52.68±0.66                   & 64.38±0.08               & 85.76±0.02                & 73.45±0.12                   \\
\multicolumn{1}{c|}{} & GRACE-CS      & 61.62±0.18                 & 64.10±0.07                   & 61.48±0.04               & 85.82±0.01                & 74.40±0.03                   \\
\multicolumn{1}{c|}{} & GRACE-C       & 62.24±0.17                 & 54.44±0.09                   & 72.74±0.05               & 85.45±0.00                & 73.05±0.56                   \\
\multicolumn{1}{c|}{} & GRACE-P       & 63.28±0.18                 & 50.66±0.05                   & 54.30±0.00               & \textbf{87.55±0.04}       & 72.17±0.02                   \\
\multicolumn{1}{c|}{} & GRACE-Com     & 58.80±0.28                 & 52.24±0.05                   & 49.26±0.05               & 84.07±0.00                & 71.57±0.16                   \\ \cmidrule{2-7}
\multicolumn{1}{c|}{} & GTOT-PM  & 69.86±0.57 & 60.20±1.66 & 63.14±1.25 & 78.92±2.99 & 65.69±1.67 \\
\multicolumn{1}{c|}{} & GTOT-CS  & 68.76±1.42 & 60.98±1.83 & 61.62±0.99 & 77.75±1.59 & 65.70±1.71 \\
\multicolumn{1}{c|}{} & GTOT-C   & 69.04±0.96 & 61.52±1.03 & 61.63±0.89 & 79.64±1.40 & 66.53±1.12 \\
\multicolumn{1}{c|}{} & GTOT-P   & 68.98±0.72 & 58.42±3.27 & 61.64±1.56 & 79.88±1.43 & 65.16±1.11 \\
\multicolumn{1}{c|}{} & GTOT-Com & 68.96±1.67 & 60.50±4.51 & 61.65±0.86 & 80.16±0.76 & 69.31±0.97 \\ \cmidrule{2-7}
\multicolumn{1}{c|}{} & AdapterGNN-PM  & 69.44±0.59 & 54.94±0.72 & 60.04±2.09 & 87.03±0.41 & 74.21±0.55 \\
\multicolumn{1}{c|}{} & AdapterGNN-CS  & 61.14±1.03 & 60.84±0.41 & 58.06±1.16 & 87.58±0.44 & 72.41±0.53 \\
\multicolumn{1}{c|}{} & AdapterGNN-C   & 63.76±0.38 & 53.40±1.70 & 69.52±1.14 & 86.76±0.14 & 70.85±0.47 \\
\multicolumn{1}{c|}{} & AdapterGNN-P   & 66.56±1.13 & 51.42±1.02 & 54.84±1.76 & 86.63±0.40 & 70.96±0.61 \\
\multicolumn{1}{c|}{} & AdapterGNN-Com & 60.36±0.89 & 52.86±0.49 & 51.54±2.14 & 86.93±0.18 & 69.95±0.80 \\ \cmidrule{2-7}
\multicolumn{1}{c|}{} & GraphControl-PM & 69.06±0.67 & 55.60±1.62 & 65.20±1.13 & 85.00±0.89 & 75.53±1.16 \\
\multicolumn{1}{c|}{} & GraphControl-CS  & 63.72±0.77 & 64.38±0.35 & 63.62±1.55 & 84.64±1.08 & 74.42±1.06 \\
\multicolumn{1}{c|}{} & GraphControl-C  & 67.14±1.06 & 55.06±1.33 & 73.24±1.20 & 84.72±0.94 & 73.82±0.72 \\
\multicolumn{1}{c|}{} & GraphControl-P & 65.54±0.48 & 53.30±1.20 & 56.36±0.85 & 87.36±1.09 & 72.32±1.02 \\
\multicolumn{1}{c|}{} & GraphControl-Com  & 58.94±0.39 & 56.78±1.34 & 51.32±1.03 & 83.86±0.69 & 72.08±1.36 \\ \cmidrule{2-7}
\multicolumn{1}{c|}{} & GraphLoRA-PM  & \uline{72.56±0.96}           & \textbf{72.52±2.50}          & \textbf{77.16±0.62}      & 86.05±0.19                & 75.57±0.64                   \\
\multicolumn{1}{c|}{} & GraphLoRA-CS  & \textbf{73.54±1.66}        & \uline{71.94±2.46}             & 76.98±0.32               & 86.51±0.66                & 75.90±0.45                   \\
\multicolumn{1}{c|}{} & GraphLoRA-C   & 71.34±0.40                 & 71.78±2.11                   & \uline{77.00±0.25}         & 86.24±0.13                & 75.47±0.35                   \\
\multicolumn{1}{c|}{} & GraphLoRA-P   & 69.02±0.72                 & 71.32±1.91                   & 74.38±1.31               & 86.57±0.71                & 74.63±1.38                   \\
\multicolumn{1}{c|}{} & GraphLoRA-Com & 69.30±1.68                 & 70.72±0.64                   & 73.70±2.00               & 86.34±0.81                & 73.81±1.24                   \\ \bottomrule
\end{tabular}
\end{table*}

\begin{table*}[!t]
\caption{Ablation experiment results in the 5-shot setting. The best experimental results are highlighted in bold.}
\label{tab:ablation-5-shot}
\begin{tabular}{@{}c|ccccc@{}}
\toprule
Variants  & PubMed              & CiteSeer            & Cora                & Photo               & Computer           \\ \midrule
w/o mmd  & 72.12±0.73          & 71.18±2.38          & 76.02±0.85          & 85.83±0.18          & 75.24±0.20        \\
w/o smmd  & 69.70±0.14          & 69.00±5.13          & 74.96±1.72          & 85.26±0.35          & 74.69±0.82          \\
w/o cl   & 71.28±0.79          & 70.76±4.99          & 75.90±0.49          & \textbf{86.07±0.50} & 75.05±0.97          \\
w/o str   & 72.66±0.78          & 66.08±0.58          & 65.84±1.30          & 85.73±0.46          & 74.63±1.18          \\
w/o lrd  & 71.06±0.88          & 70.84±1.57          & 74.68±1.24          & 86.06±0.17          & 74.44±0.36          \\
w/o nfa   & 69.82±0.26          & 33.94±0.21          & 69.52±0.28          & 77.14±0.99          & 73.53±0.35          \\
w/o sktl   & 71.14±0.61          & 64.40±4.87          & 75.78±0.30          & 85.73±0.07          & 74.91±0.57          \\ \midrule
GraphLoRA & \textbf{72.56±0.96} & \textbf{72.52±2.50} & \textbf{77.16±0.62} & 86.05±0.19          & \textbf{75.57±0.64} \\ \bottomrule
\end{tabular}
\end{table*}

\section{Baselines}
\label{appendix:baseline}
\begin{itemize}
  \item \textbf{GCN}~\cite{DBLP:conf/iclr/KipfW17}: GCN is a foundational graph neural network that effectively propagates information within the graph structure by capturing the relationships among nodes and their neighboring nodes.
  \item \textbf{GAT}~\cite{DBLP:conf/iclr/VelickovicCCRLB18}: GAT is another classic graph neural network. In contrast to GCN, GAT introduces attention mechanisms, allowing each node to dynamically adjust weights based on the importance of its neighboring nodes during the representation update process.
  \item \textbf{GRACE}~\cite{DBLP:journals/corr/abs-2006-04131}: GRACE adopts the SimCLR framework~\cite{DBLP:conf/icml/ChenK0H20} and incorporates two strategies to augment the source graph. It aims to maximize the mutual information between two views by enhancing agreement at the node level.
  \item \textbf{COSTA}~\cite{DBLP:conf/kdd/ZhangZSKK22}: COSTA introduces a feature augmentation framework to perform augmentations on the hidden features, mitigating the issue of highly biased node embeddings obtained from graph enhancement. Moreover, it accelerates the speed of graph contrastive learning. 
  \item \textbf{CCA-SSG}~\cite{DBLP:conf/nips/ZhangWYWY21}: CCA-SSG proposes an innovative feature-level optimization objective based on Canonical Correlation Analysis for graph contrastive learning, presenting a conceptually simple yet effective model.
  \item \textbf{HomoGCL}~\cite{DBLP:conf/kdd/LiWXL23a}: HomoGCL enhances graph contrastive learning by leveraging the homophily of the graph. It directly utilizes the homophily of the graph by estimating the probability of neighboring nodes being positive samples via a Gaussian Mixture Model.
  \item \textbf{GPPT}~\cite{DBLP:conf/kdd/SunZHWW22}: GPPT is a graph prompt learning method that introduces a novel paradigm for graph neural network transfer learning known as "pre-train, prompt, fine-tune", designed specifically for cross-task transfer learning.
  \item \textbf{GPF}~\cite{DBLP:conf/nips/FangZYWC23}: GPF is a universal prompt-based tuning method for pre-trained GNN models, theoretically achieving the same effect as any form of prompting function.
  \item \textbf{GraphPrompt}~\cite{DBLP:conf/www/LiuY0023}: GraphPrompt introduces a unification framework by mapping different tasks to a common task template and proposes a learnable task-specific prompt vector to guide each downstream task in fully leveraging the pre-trained model.
  \item \textbf{ProG}~\cite{DBLP:conf/kdd/SunCLLG23}: ProG reformulates different-level tasks into unified ones and designs a multi-task prompting method for graph models.
  \item \textbf{GRACE$_t$}: GRACE$_t$ is a variant of GRACE that involves pre-training on the source graph using the GRACE method and fine-tuning on the target graph.
  \item \textbf{GTOT}~\cite{DBLP:conf/ijcai/ZhangXHRB22}: GTOT is an optimal transport-based fine-tuning method. It formulates graph local knowledge transfer as an optimal transport problem, preserving the local information of the fine-tuned network from pre-trained models.
  \item \textbf{AdapterGNN}~\cite{DBLP:conf/aaai/LiH024}: AdapterGNN is a parameter-efficient fine-tuning method, which freezes the pre-trained network and introduces adapters to it.
  
  \item \textbf{GraphControl}~\cite{DBLP:journals/corr/abs-2310-07365}: GraphControl is a recent research endeavor in the field of graph neural network transfer learning. Drawing inspiration from ControlNet~\cite{DBLP:conf/iccv/ZhangRA23}, it incorporates its core concepts to enhance transfer learning in graph neural networks.
\end{itemize}

\section{Additional Experimental Results}
\label{appendix:add_exp}

\subsection{Performance on Heterogeneous Graphs}
\label{appendix:hete_result}
We evaluate GraphLoRA on two heterogeneous graphs, Squirrel and Chameleon~\cite{DBLP:conf/iclr/PeiWCLY20}, with results presented in \tablename~\ref{tab:hete}. The results show that GraphLoRA does not perform the best on heterogeneous graphs, which may be due to the Structure-aware Regularization module leveraging the homophily property of homogeneous graphs, making it less suitable for heterogeneous graphs. Nonetheless, GraphLoRA ranked 6th and 3rd among 14 methods on the two datasets, respectively, indicating respectable performance.

\subsection{Performance with Different Pretraining Methods}
\label{appendix:diff_pretrain}
We evaluate GraphLoRA with different pretraining methods, including CCA-SSG and HomoGCL, with results in \tablename~\ref{tab:diff_pretrain}. GraphLoRA performs best in most cases, achieving 13.66\% better than CCA-SSG and 4.05\% better than HomoGCL on average, demonstrating its effectiveness across various pretraining methods.

\subsection{Performance in the 5-shot Setting}
\label{appendix:5_shot_result}
The results of the comparison experiment in the 5-shot setting are shown in~\tablename~\ref{tab:5shot}. The results of the ablation experiment in the 5-shot setting are shown in~\tablename~\ref{tab:ablation-5-shot}.

\end{document}